\documentclass[letterpaper, 10 pt, conference]{ieeeconf}  

\IEEEoverridecommandlockouts                              
\usepackage{microtype}
\usepackage{graphicx}
\usepackage{subfigure}
\usepackage{booktabs} 

\usepackage[pdfencoding=auto, psdextra, pagebackref=true]{hyperref}
\usepackage{xcolor}
\hypersetup{
    colorlinks,
    linkcolor={blue!50!black},
    citecolor={blue!50!black},
    urlcolor={blue!80!black}
}

\usepackage{algorithm}
\usepackage{algorithmic}
\usepackage{forloop}

\usepackage{amsthm}
\usepackage{amsmath}
\usepackage{amssymb}
\usepackage{multirow} 
\usepackage{adjustbox} 

\usepackage[capitalize,noabbrev]{cleveref}

\newcommand{\abs}[1]{\left| #1 \right|}

\newcommand{\R}{{{\mathbb R}}}

\DeclareMathOperator*{\argmax}{argmax}

\renewcommand{\epsilon}{\varepsilon}
\newcommand{\ps}[1]{\langle #1 \rangle}

\newcommand{\bE}{{{\mathbb E}}}
\newcommand{\bN}{{{\mathbb N}}}

\renewcommand{\O}{\mathcal{O}}

\renewcommand{\S}{\mathcal{S}}
\newcommand{\A}{\mathcal{A}}

\newcommand{\mn}{\mathcal{N}}
\newcommand{\ms}{\mathcal{S}}
\newcommand{\ma}{\mathcal{A}}
\newcommand{\mai}{\mathcal{A}_i}
\newcommand{\E}{\mathbb{E}}
\newcommand{\Prob}{\mathbb{P}}

\DeclareMathOperator{\dom}{dom}

\usepackage{xcolor}

\theoremstyle{plain}
\newtheorem{theorem}{Theorem}[section]
\newtheorem{proposition}[theorem]{Proposition}
\newtheorem{lemma}[theorem]{Lemma}

\theoremstyle{definition}

\newtheorem{assumption}[theorem]{Assumption}
\theoremstyle{remark}
\newtheorem{remark}[theorem]{Remark}

\usepackage[textsize=tiny]{todonotes}

\title{\LARGE \bf
Independent Policy Mirror Descent for Markov Potential Games:\\ Scaling to Large Number of Players
}

\author{Pragnya Alatur$^*$  \qquad Anas Barakat$^*$ \qquad Niao He 
\thanks{Authors are affiliated with the department of Computer Science, ETH Zurich, Switzerland. $*$ stands for first-author contribution. }%
}

\begin{document}

\maketitle
\thispagestyle{empty}
\pagestyle{empty}

\begin{abstract}
Markov Potential Games (MPGs) form an important sub-class of Markov games, which are a common framework to model multi-agent reinforcement learning problems.
In particular, MPGs include as a special case the identical-interest setting where all the agents share the same reward function.  
Scaling the performance of Nash equilibrium learning algorithms to a large number of agents is crucial for multi-agent systems. 
To address this important challenge, we focus on the independent learning setting where agents can only have access to their local information to update their own policy. 
In prior work on MPGs, the iteration complexity for obtaining $\epsilon$-Nash regret scales \textit{linearly} with the number of agents~$N$.
In this work, we investigate the iteration complexity of an independent policy mirror descent~(PMD) algorithm for MPGs. 
We show that PMD with KL regularization, also known as natural policy gradient, enjoys a better~$\sqrt{N}$ dependence on the number of agents, improving over PMD with Euclidean regularization and prior work. Furthermore, the iteration complexity is also independent of the sizes of the agents' action spaces. 
\end{abstract}


\section{Introduction}

Introduced by the seminal work of \cite{shapley53}, Markov games (also called stochastic games) offer a convenient mathematical framework to model multi-agent reinforcement learning (MARL) problems where several agents interact in a shared space environment to make strategic decisions. 
A major challenge in MARL is to deal with the so-called curse of multi-agents in which the size of the joint action space scales exponentially with the number of agents. 
Given the ever-growing scale of real-world multi-agent systems such as transportation networks, social networks and autonomous vehicles, scalability of MARL algorithms to these large-scale systems involving a large number of agents poses a fundamental challenge. 
A useful approach to address this scaling issue in the literature consists in adopting the independent learning protocol in which each agent independently selects their policy and observes only local information including their own actions and rewards, along with a shared state (see e.g. \cite{ozdaglar-sayin-zhang21survey} for a recent survey). 

In this work, we focus on designing scalable and independent  algorithms for a specific subclass of Markov games: Markov Potential Games (MPGs). 
Inspired by the celebrated (static) potential games in normal form pioneered by the celebrated work \cite{monderer-shapley96}, MPGs can be seen as an extension of this class of games to the dynamic setting in which the environment is subject to state transitions. 
This class of MPGs has recently attracted attention in a line of works 
\cite{narasimha-et-al22,leonardos2021global,zhang-et-al21grad-play, zhang-et-al22softmax,ding2022independent,jordan-barakat-he24}. In particular, MPGs have been used for modeling marketplaces for distributed energy resources such as electricity markets and pollution tax \cite{narasimha-et-al22,jordan-barakat-he24}. 

A common objective in MARL problems is to reach a Nash equilibrium. 
It is well-known that even in (static) normal form games, computing a Nash equilibrium is computationally hard in general~\cite{daskalakis2009complexity}\footnote{More precisely, \cite{daskalakis2009complexity} showed that computing a Nash equilibrium is in the complexity class PPAD.}. For MPGs, however, Nash equilibria are computationally tractable \cite{leonardos2021global, song-mei-bai22}. 
Recently, in the static (stateless) setting, it has been shown in~\cite{cen2022independent} that the dependence on the number of players of the Nash regret guarantees can be improved to be sublinear using an independent natural policy gradient algorithm with additional entropic regularization. 

In the present work, we investigate the performance of an independent policy mirror descent algorithm for finding approximate Nash equilibria in MPGs. 
In particular, we focus on improving the dependence on the number of players of the Nash regret incurred by the algorithm. Our contributions are summarized as follows:
  \begin{itemize}

        \item We introduce a Policy Mirror Descent (PMD) algorithm which can be implemented independently by the agents, inspired by prior work in the single-agent setting \cite{lan23pmd,xiao22jmlr}. We show that this algorithm can be seen as a mirror descent algorithm with a dynamically weighted Bregman divergence regularization. This observation allows to draw a direct connection between PMD and the standard independent policy gradient algorithm. This algorithm unifies the Q-ascent algorithm obtained when using Euclidean regularization and the celebrated Natural gradient algorithm which follows from using a KL regularization.  

        \item We establish Nash regret bounds for PMD with either Euclidean or KL regularization for MPGs. In particular, we show that the dependence w.r.t. the number of players in the iteration complexity to reach an~$\epsilon$-Nash regret\footnote{The iteration complexity to obtain $\epsilon$-Nash regret gives us the iteration complexity required to obtain an $\epsilon$-approximate Nash equilibrium.} improves from~$N$ to~$\sqrt{N}$  when using the KL regularization instead of the Euclidean regularization. Notably, the Nash regret bound is also independent of the size of the agents' action spaces. To achieve these improvements over prior work, the use of the geometry induced by entropic regularization appears to be crucial. Our guarantees hold in the full information setting where we suppose access to the average Q-values w.r.t. the policies of the opponents.
    \end{itemize}

\section{Related Work}

For static games, a large body of work in the literature (that we cannot hope to give justice to here) has been interested in establishing guarantees for learning algorithms in potential games for concave potential functions. Beyond this particular concavity structure, a few works have addressed the question of the convergence of some no-regret algorithms such as the celebrated Hedge algorithm to Nash equilibria including in the bandit feedback setting~\cite{heliou-et-al17}. Our work has been partially inspired by \cite{cen2022independent} which provides a~$\tilde{O}(\min(\sqrt{N}, \phi_{\max}) \phi_{\max}/\epsilon^{4})$ iteration complexity (where~$\phi_{\max}$ is the maximum of the potential function) for reaching an~$\epsilon$-approximate Nash equilibrium using an independent natural policy gradient algorithm with additional entropy regularization. While this result does not require stationary policies to be isolated (see Assumption~\ref{hyp:isolated-fixed-points}) and improves over prior work in terms of dependence on the number of agents, it only holds for static potential games and achieves a worse iteration complexity in terms of accuracy~$\epsilon$ compared to the existing~$\mathcal{O}(\epsilon^{-2})$ results for MPGs.

Initiated by works such as \cite{marden12} proposing state-based potential games, a line of research has focused on MPGs as a natural generalization of potential games to the dynamic stateful setting as well as (single-agent) Markov Decision Processes to the multi-agent setting. \cite{leonardos2021global, zhang-et-al21grad-play} proposed an independent policy gradient method for MPGs with a~$\mathcal{O}(\epsilon^{-2})$ iteration complexity to reach an~$\epsilon$-approximate Nash equilibrium. Later, \cite{ding2022independent} proposed the Q-ascent algorithm for MPGs with an~$\mathcal{O}(\epsilon^{-2})$ iteration complexity improving over the dependence on the state space size. A few other works focused on asymptotic convergence results~\cite{maheshwari-et-al22,fox-et-al22indep-npg-asymp} or regret guarantees for policy gradient algorithms with softmax parametrization \cite{fox-et-al22indep-npg-asymp,zhang-et-al22softmax}. 
To the best of our knowledge, all the aforementioned finite-time guarantees in prior works provide iteration complexities scaling at least linearly with the number of agents. The main focus of our work is to reduce such a dependence. We refer the reader to Table~\ref{tab:related-work} for a brief comparison of our results to the closest related works.
More recently, \cite{sun-et-al23provably} established an \textit{asymptotic}~$\mathcal{O}(\epsilon^{-1})$ iteration complexity result for the independent natural policy gradient under some specific additional assumptions that we do not require in our work (see their Assumption 3.2), i.e. a bound that is valid only after a certain unknown number of iterations. In contrast, our guarantees hold globally for any iteration.

Sample-based approaches for MPGs have been proposed and analyzed in \cite{leonardos2021global,zhang-et-al21grad-play,ding2022independent} for example. 

Finally, a few works proposed algorithms for extensions of MPGs to networked settings \cite{aydin-eksin23networked,aydin-eksin23timevarying,zhou-et-al23nmpgs} or to a more general class of games relaxing the definition of MPGs~\cite{guo-et-al23alpha-mpgs}.

\begin{table}[h]
\caption{Iteration complexity for obtaining $\epsilon$-Nash regret in MPGs. 
}
\centering
\begin{adjustbox}{width=\columnwidth}
\begin{tabular}{|c|c|}
\hline
 Algorithm & $\epsilon$-Nash Regret\\ 
\hline
 Independent PG \footnotemark[1] \cite{leonardos2021global,zhang-et-al21grad-play} & $\O \left( \frac{\phi_{\max}\kappa_{\rho}^2 S N \underset{i \in \mn}{\max} |\mai|}{(1-\gamma)^4 \epsilon^2} \right)$\\
 \hline
 Q-Ascent\footnotemark[1]\cite{ding2022independent} & $\O \left( \frac{\Phi_{\max}\min\{\tilde{\kappa}_{\rho}, S\}^4 N \underset{i \in \mn}{\max} |\mai|}{(1-\gamma)^6 \epsilon^2} \right)$\\
 \hline
 GD (softmax) \cite{zhang-et-al22softmax} & $\O \left( \frac{\phi_{\max} M^2 N \underset{i \in \mn}{\max} |\mai|}{(1-\gamma)^4 c^2 \epsilon^2} \right)$\\
 \hline
 Natural GD (softmax) \cite{zhang-et-al22softmax} & $\O \left(\frac{ \phi_{\max}^2 M  N}{(1-\gamma)^3 c \epsilon^2} \right)$\\
  \hline
  \textbf{PMD (Euclidean reg.)} & $\O \left( \frac{\phi_{\max}^2\tilde{\kappa}_{\rho} \sum_{i=1}^N |\mai|}{(1-\gamma)^4 \epsilon^2} \right)$\\
  \hline
  \textbf{PMD  (KL reg.)}\footnotemark[2] & $\O \left( \frac{\phi_{\max}^2 \tilde{\kappa}_{\rho}\sqrt{N}}{(1-\gamma)^4 c \epsilon^2} \right)$\\
  \hline
\end{tabular}
\end{adjustbox}
\begin{minipage}{8.2cm}
\vspace{0.2cm}
\small 
We highlight the dependence on the number of agents $N$ and focus only on prior work with non-asymptotic guarantees.
Notation: $\rho \in \Delta(\ms)$ is the initial distribution, $\Phi_{\max} \triangleq \max_{\pi\in\Pi} \abs{\Phi^{\pi}(\rho)}$, $c$ is a constant dependent on the initial policy and defined in Theorem~\ref{thm:exact-pmd-natural-gradient}, $M \triangleq  \sup_{\pi} \max_{s \in \mathcal{S}} \frac{1}{d_{\rho}^{\pi}(s)} > 0$ under Assumption~\ref{hyp:positive-discounted-visit-distrib}. With our definition of MPGs, $\Phi_{\max} \leq \frac{\phi_{\max}}{1-\gamma}\,.$ For other notations see Preliminaries section \ref{sec:preliminaries} below. \footnotemark[1]{Use a more general definition of MPGs, see Remark~\ref{rem:def-mpgs} below for a discussion.}\footnotemark[2]{We also obtain a~$\O \left( \frac{\phi_{\max}^2 M\sqrt{N}}{(1-\gamma)^3 c \epsilon^2} \right)$ regret bound for this algorithm under minor modifications in our proof. This shows a strict improvement in terms of~$N$ w.r.t. the result for Natural GD \cite{zhang-et-al22softmax} in the table.}
\vspace{-0.5cm} 
\end{minipage}
\label{tab:related-work}
\end{table}

\section{Preliminaries}
\label{sec:preliminaries}

\noindent\textbf{Markov Game.} We consider a discounted infinite time horizon Markov Game~$\Gamma = \left\langle \mn, \ms, (\mai)_{i \in \mn}, P, (r_i)_{i \in \mn}, \rho, \gamma \right\rangle$ \cite{shapley53} where: 
\begin{itemize}
\item $\mn \triangleq \{1, \cdots, N\}$ with~$N \geq 2$ is the set of players\footnote{We use agent and player interchangeably throughout this work.}, 
\item $\ms$ is a finite set of states\footnote{Extension to infinite countable or even continuous state space is possible under additional technical assumptions such as  boundedness of the reward and potential functions.}, 
\item $\ma_i$ is a finite set of actions for each~$i \in \mn$. We denote by~$\ma\triangleq\times_{i\in\mn}\ma_i$ the set of joint actions,
\item $P: \ms \times \ma \to \Delta(\ms)$ is the Markov transition kernel, where we use throughout this paper the notation~$\Delta(\ms)$ for the set of probability distributions over the set~$\ms$, 
\item $r_i: \ms \times \ma \to \R$ is the reward function of agent~$i \in \mn$\,,
\item $\rho \in \Delta(\ms)$ is an initial distribution over states, 
\item $\gamma \in (0,1)$ is a discount factor\,. 
\end{itemize}

\noindent\textbf{Policies.} For each~$i \in \mn$, denote by~$\pi_i: \ms \to \Delta(\mai)$ the policy of agent~$i$ and~$\Pi^i$ the set of all possible Markov stationary policies for agent~$i\in \mn\,.$ We denote by~$\pi = (\pi_i)_{i \in \mn}$ the joint policy. 
We use the notation~$\Pi \triangleq \Large{\times}_{i \in \mn} \Pi^i$ for the set of joint Markov stationary policies.\\ 

In a Markov game, the agents' interaction with the environment unfolds as follows: At each time step~$t \geq 0$, the agents observe a shared state~$s_t \in \ms$ and choose a joint action~$a_t = (a_{t,i})_{i \in \mn}$ according to their joint policy~$\pi^{(t)}$, i.e. for every~$i \in \mn, a_{t,i}$ is sampled according to~$\pi_i^{(t)}(\cdot|s_t)$. Each agent~$i \in \mn$ receives a reward~$r_i(s_t, a_t)$. Then the game proceeds by transitioning to a state~$s_{t+1}$ drawn from the distribution~$P(s_t, a_t)\,.$\\

\noindent\textbf{Occupancy measures and distribution mismatch coefficients.}
For any joint policy~$\pi \in \Pi,$ we define the state occupancy measure~$d_{\rho}^{\pi}$ for every state~$s \in \ms$ by~$d_{\rho}^{\pi}(s) = (1-\gamma) \sum_{t=0}^{\infty} \gamma^t \Prob_{\rho,\pi}(s_t = s)\,.$ Similarly to prior work \cite{ding2022independent}, we introduce distribution mismatch coefficients to quantify the difficulty for learning agents to explore in a Markov game. In particular, we define for any distribution~$\rho \in \Delta(\ms)$ and any policy~$\pi \in \Pi$, the distribution mismatch coefficient~$\kappa_{\rho} \triangleq \sup_{\pi \in \Pi} \|\frac{d_{\rho}^{\pi}}{\rho}\|_{\infty}$ as well as its minimax version~$\tilde{\kappa}_{\rho} \triangleq \inf_{\nu \in \Delta(\ms)} \sup_{\pi \in \Pi} \|\frac{d_{\rho}^{\pi}}{\nu}\|_{\infty}\,,$ where the division is to be understood componentwise in both cases. Notice that~$\tilde{\kappa}_{\rho} \leq \min(\kappa_{\rho}, |\ms|)\,.$\\

\noindent\textbf{Value functions.} 
For each joint policy~$\pi \in \Pi,$ and each~$i \in \mn$, define the state-action value function~$Q_i^{\pi}:\ms\times\ma \rightarrow \R$ for each~$s \in \ms, a \in \ma$ by 
\begin{equation*}
    Q_i^{\pi}(s,a) \triangleq \E_{\pi}\left[\sum_{t=0}^{\infty} \gamma^t r_i(s_t, a_t) | s_0 = s, a_0 = a \right]\,.
\end{equation*}
We also define the value function~$V_i^{\pi}: \ms \to \R$ for each~$s \in \ms$ by~$V_i^{\pi}(s) \triangleq \sum_{a \in \ma} \pi(a|s)  Q_i^{\pi}(s,a)\,.$
For any initial state distribution~$\rho$ and any policy~$\pi \in \Pi$, we use the notation~$V_i^{\pi}(\rho) \triangleq \E_{s \sim \rho}[V_i^{\pi}(s)]\,.$\\ 

\noindent\textbf{Nash equilibria.}
For any~$\epsilon \geq 0$, a joint policy~$\pi = (\pi_i)_{i \in \mn} \in \Pi$ is an $\epsilon$-approximate Nash equilibrium for the game~$\Gamma$ if for every~$i \in \mn$ and every~$\pi_i' \in \Pi^i$, 
$V_i^{\pi_i,\pi_{-i}}(\rho) \geq V_i^{\pi_i',\pi_{-i}}(\rho) - \epsilon$. 
When~$\epsilon = 0$, such a policy~$\pi$ from which no agent has an incentive to deviate unilaterally, is called a Nash equilibrium policy.\\ 

\noindent\textbf{Markov Potential Game.} A Markov game~$\Gamma$ is a Markov Potential Game (MPG) if there exists a function~$\phi: \ms \times \ma \to \R$ s.t. for any~$i \in \mn$, and any policies~$(\pi_i', \pi_{-i}), \pi \in \Pi$, 
\begin{equation}
 V_i^{\pi'_i,\pi_{-i}}(\rho) - V_i^{\pi_i,\pi_{-i}}(\rho) = \Phi^{\pi'_i,\pi_{-i}}(\rho) - \Phi^{\pi_i,\pi_{-i}}(\rho), 
\end{equation}
where $\Phi^{\pi}(\rho) \triangleq \E_{s_0\sim \rho, \pi} \left[ \sum_{t=0}^{\infty} \gamma^t \phi(s_t,a_t) \right]$ is the so-called \emph{total potential function} induced by the function~$\phi\,.$ 
We denote by $\phi_{\max} \triangleq \max_{s \in \ms, a \in \ma} |\phi(s,a)|\,.$ 
The identical-interest case is an important particular case of this definition. In this case in which all the reward functions are identical, it can be easily seen that the total potential function is the value function of any of the players. 
Beyond this important case, conditions to obtain MPGs with non-identical rewards were identified in \cite{narasimha-et-al22,leonardos2021global,zhang-et-al21grad-play}. 

\begin{remark}
\label{rem:def-mpgs}
The definition of MPGs we consider coincides with the definitions used in~\cite{zhang-et-al21grad-play,zhang-et-al22softmax}. 
Compared to other works \cite{leonardos2021global,ding2022independent}, our definition requires the total potential function to have a cumulative discounted sum form with a function $\phi$. Such a sum structure naturally connects MPGs to their historical parent, the (stateless) potential game. Notice that this definition also captures the identical-interest setting in the same way. 
 We will also comment later on about the consequences of this definition in our results. 
\end{remark}

\section{Independent Policy Mirror Descent}

Inspired by a line of works \cite{lan23pmd,xiao22jmlr,zhan-et-al23pmd-reg} for the single-agent setting, we extend the Policy Mirror Descent~(PMD) algorithm to the multi-agent setting for MPGs. 
Then we discuss two instantiations of PMD with Euclidean and KL regularizations respectively.

We start by defining a central quantity for independent learning in MPGs: 
The averaged Q-value $\bar{Q}_i^{\pi}: \ms \times \mai \to \R$ for any agent $i \in \mn$ and policy $\pi \in \Pi$ defined as follows:
    \begin{equation*}
        \bar{Q}_i^{\pi}(s,a_i) \triangleq \bE_{a_{-i} \sim \pi_{-i}(\cdot|s)}[Q_i^{\pi}(s,a_i,a_{-i})], \forall s \in \ms, a_i \in \mai\,.
    \end{equation*}

For each agent~$i \in \mn$ and each state~$s \in \ms$, starting from an initial policy~$\pi_i^{(1)} \in \Delta(\mathcal{A}_i)^{\ms}$,  independent PMD for MPGs unfolds iteratively as follows: 
\begin{equation}
    \pi_{i,s}^{(t+1)} \in \underset{\pi_{i,s} \in \Delta(\mathcal{A}_i)}{\mathrm{argmax}} \left\{ \langle \bar{Q}_{i,s}^{\pi^{(t)}}, \pi_{i,s} \rangle - \frac{1}{\eta} D_{\psi}(\pi_{i,s}, \pi_{i,s}^{(t)})\right\}\,,
    \tag{PMD}
\end{equation}
where we use the shorthand notations~$\pi_{i,s}^{(t)} = \pi_{i}^{(t)} (\cdot | s) $ and~$\pi_{i,s} = \pi_i (\cdot | s)$, where~$\eta$ 
is a positive step size, $\bar{Q}_{i,s}^{\pi^{(t)}} \in \mathbb{R}^{|\mai|}$ is a vector containing average state-action values for $\pi^{(t)}$ at state~$s$, and $D_{\psi}$ is the Bregman divergence induced by a mirror map~$\psi: \dom \psi \to \R$ such that~$\Delta(\mai) \subset \dom \psi\,,$ i.e. for any~$p, q \in \Delta(\mai),$
\begin{equation*}
D_{\psi}(p, q) = \psi(p) - \psi(q) - \ps{\nabla \psi(q), p - q}\,, 
\end{equation*}
where we suppose throughout this paper that the function~$\psi$ is of Legendre type, i.e. strictly convex and essentially smooth in the relative interior of~$\dom \psi$ (see~\cite{rockafellar-et-al70cvx-analysis}, section~26). We emphasize that the updates are performed \emph{simultaneously} by all the agents.

\begin{remark}
A standard mirror descent algorithm as in the optimization literature would feature the gradients~$\nabla V_i^{\pi}(\rho)$ of the value functions in the inner product in the (PMD) update rule above. Recall that the policy gradients are given by~$\frac{\partial V_i^{\pi}(\rho)}{\partial \pi_i(a_i|s)} = \frac{1}{1-\gamma} d_{\rho}^{\pi}(s) \bar{Q}_i^{\pi}(s,a_i)$ for every~$s \in \ms, a_i \in \mai\,$ (e.g. see \cite{xiao22jmlr}).
Similarly to the PMD algorithm in the single-agent setting, (PMD) corresponds to a standard mirror descent algorithm enhanced with adaptive preconditioning with a dynamically weighted divergence: Each state Bregman divergence is weighted by the discounted visitation measure~$d_{\rho}^{\pi}(s)$\,. We refer the reader to section 4, p.15-16 in \cite{xiao22jmlr} for a more precise exposition of this observation in the single-agent setting.  
\end{remark}

The mirror map choice generates a large class of PMD algorithms. 
In this work, we focus on two concrete algorithms corresponding to the choice of two prominent mirror maps: 
\begin{enumerate}
\item \textbf{Projected $Q$-ascent.} When the mirror map~$\psi$ is the squared $\ell_2$-norm, the corresponding Bregman divergence is the squared Euclidean distance and the resulting algorithm is playerwise projected $Q$-ascent:  
\begin{equation}
\label{eq:pmd-euclidean}
\pi_{i,s}^{(t+1)} = \text{Proj}_{\Delta(\mai)}(\pi_{i,s}^{(t)} + \eta  \bar{Q}_{i,s}^{\pi^{(t)}})\,,
\end{equation}
for every~$i \in \mn, s \in \ms$ and~$\text{Proj}_{\Delta(\mai)}$ is the projection operator on the simplex~$\Delta(\mai)\,.$\\

\item \textbf{Natural Policy Gradient.} When the mirror map~$\psi$ is the negative entropy, the Bregman divergence~$D_{\psi}$ is the Kullback-Leibler~(KL) divergence and~(PMD) boils down to independent Natural Policy Gradient which updates for every~$i \in \mn, (s, a_i) \in \ms \times \mai$ as follows, 
\begin{equation}
\label{eq:pmd-KL}
\pi_i^{(t+1)}(a_i|s) =  \frac{\pi_i^{(t)}(a_i|s) \exp(\eta \bar{Q}_{i}^{\pi^{(t)}}(s,a_i))}{Z_t^{i,s}}\,, 
\end{equation}
where~$Z_t^{i,s}\triangleq \sum_{a_i \in \mai} \pi_i^{(t)}(a_i|s) \exp(\eta\, \bar{Q}_{i}^{\pi^{(t)}}(s,a_i))\,,$ and the initial joint policy~$\pi^{(0)}$ is chosen in the interior of~$\Delta(\ma)^{|\ms|}.$
Algorithm~\eqref{eq:pmd-KL} is also  referred to as a multiplicative weights update algorithm. We remark that the averaged Q-function in~\eqref{eq:pmd-KL} (including in~$Z_t^{i,s}$) can be replaced by the averaged advantage function~$\bar{A}_i^{\pi^{(t)}}$ defined for every~$s \in \ms, i \in \mn, a_i \in \mai$ by
\begin{equation}
\bar{A}_i^{\pi^{(t)}}(s,a_i) \triangleq \bar{Q}_{i}^{\pi^{(t)}}(s,a_i) - V_i^{\pi^{(t)}}(s)\,.
\end{equation}

\end{enumerate}

\noindent\textbf{Information setting.} Throughout this work, we assume that each agent~$i \in \mn$ has access to their average Q-value function~$\bar{Q}_{i}^{\pi^{(t)}}$ at each iteration~$t$ of the algorithm. 
Beyond this oracle-based feedback setting, the averaged Q-value functions can also be estimated independently. Indeed, each agent can resort to payoff-based methods to estimate their own average Q-value function relying uniquely on their received rewards and without observing the policies or the actions of the other agents. 
The extension of our guarantees to this stochastic setting requires further investigation, especially when using the KL regularization. We leave it for future work.

\section{Nash Regret Analysis}

In this section, we provide Nash regret guarantees for the PMD algorithm we introduced in the previous section when implemented either with Euclidean~\eqref{eq:pmd-euclidean} or KL~\eqref{eq:pmd-KL} regularization, with a focus on the dependence on the number of players. 
Following prior work (see e.g. \cite{ding2022independent,zhang-et-al22softmax}), to make our analysis precise, we first define the notion of Nash regret for every time horizon~$T \geq 1$ as follows: 
\begin{equation*}
        \text{Nash-regret(T)} \triangleq \frac{1}{T} \sum_{t=1}^T \max_{i \in \mn}\max_{\pi'_i \in \Pi^i} V_{i}^{\pi'_i, \pi^{(t)}_{-i}}(\rho) - V_{i}^{\pi^{(t)}}(\rho)\,, 
    \end{equation*}
where~$\pi^{(t)} = (\pi_i^{(t)}, \pi_{-i}^{(t)})$ is the joint policy of the~$N$ players at time step~$t \in \{1, \cdots, T\}$ and~$\rho$ is the initial distribution over the state space~$\ms$.  It follows from this definition that $\text{Nash-regret(T)} \geq 0$ for every~$T \geq 1\,.$ Furthermore, when $\text{Nash-regret(T)} \leq \epsilon$ for some accuracy~$\epsilon > 0$, there exists~$t^* \in \{1, \cdots, T\}$ such that~$\pi^{(t^*)}$ is an $\epsilon$-Nash equilibrium. 
At each time step~$t$ and for every player $i\in\mn$, the joint policy~$\pi^{(t)}$ is compared to the policy where player~$i$ unilaterally deviates to its best response to policy~$\pi_{-i}^{(t)}$\,. The difference in value functions quantifies player $i$'s Nash gap. The Nash regret computes the average over the worst player's Nash gap induced by the joint policy~$\pi^{(t)}$ over the time horizon~$T\,.$ The complete proofs of the results in this section are deferred to the appendix.

\subsection{Analysis of PMD with Euclidean Regularization}

We make the following standard assumption ensuring that all the states are visited.
\begin{assumption}
\label{hyp:positive-discounted-visit-distrib}
For every joint policy~$\pi \in \Pi$ and every state~$s \in \mathcal{S}, d_{\rho}^{\pi}(s) > 0\,.$
\end{assumption}
Since $d_{\rho}^{\pi}(s) \geq (1-\gamma) \rho(s)$ for every~$\pi \in \Pi, s \in \ms\,,$ Assumption~\ref{hyp:positive-discounted-visit-distrib} is automatically satisfied when the initial distribution~$\rho$ has full support. 

The key step of our analysis consists in quantifying the potential function improvement between two consecutive time steps of Algorithm~\eqref{eq:pmd-euclidean}. Since agents are updating their policy simultaneously, the challenge is to establish a joint policy improvement. The following lemma characterizes this improvement. 

\begin{proposition}[Potential Improvement - Euclidean PMD] 
\label{prop:potential-improvement-euclidean} 
Under Assumption~\ref{hyp:positive-discounted-visit-distrib}, for any~$\mu \in \Delta(\mathcal{S}), t \geq 1,$ we have 
\begin{align*}
\Phi^{\pi^{(t+1)}}(\mu) - \Phi^{\pi^{(t)}}(\mu) \geq \left( \frac{1}{2\eta(1-\gamma)} - \frac{\phi_{\max} \sum_{i=1}^N |\mathcal{A}_i|}{(1-\gamma)^2} \right) \cdot\\  
\sum_{s \in \mathcal{S}} d_{\mu}^{\pi^{(t+1)}}(s) \sum_{i=1}^N \|\pi_{i,s}^{(t+1)} - \pi_{i,s}^{(t)}\|^2\,.
\end{align*}
\end{proposition}

\noindent\textbf{Proof sketch (of Proposition~\ref{prop:potential-improvement-euclidean}).} 
The first step consists in using the performance difference lemma (\cref{lem:performance-difference-single-agent}) to relate the potential function change to the policy change: 
\begin{align*}
&\Phi^{\pi^{(t+1)}}(\mu) - \Phi^{\pi^{(t)}}(\mu)\\ 
&=  \sum_{s \in \mathcal{S}, a \in \mathcal{A}} \frac{d_{\mu}^{\pi^{(t+1)}}(s)}{1-\gamma} (\pi^{(t+1)}(a|s) - \pi^{(t)}(a|s))Q_{\phi}^{\pi^{(t)}}(s,a)\,, 
\end{align*}
where~$Q_{\phi}^{\pi^{(t)}}$ is the $Q$-function induced by the reward function~$\phi$ and the policy~$\pi^{(t)}\,.$ The second step consists in connecting the joint policy change to the individual policy deviation to use the update rule of the algorithm. For this, we recall that the joint policies are product policies across agents and use the following decomposition: 
\begin{align*}
&\pi^{(t+1)}(a|s) - \pi^{(t)}(a|s)\\ 
&= \sum_{i=1}^N ( \pi_i^{(t+1)}(a_i|s) - \pi_i^{(t)}(a_i|s)) \,\tilde{\pi}_{-i}^{(t)}(a_{-i}|s)\,, 
\end{align*}
where $\tilde{\pi}_{-i}^{(t)}(a_{-i}|s) \triangleq \prod_{j=1}^{i-1} \pi_j^{(t+1)}(a_j|s) \prod_{j=i+1}^N \pi_j^{(t)}(a_j|s)$ with the convention~$\prod_{j=1}^{0} z_j = 1$ and~$\prod_{j=N+1}^N z_j = 1$ for any sequence of nonnegative reals~$(z_j)\,.$ The rest of the proof follows from plugging this decomposition into the first expression of the potential change above, using the update rule of the algorithm to obtain a potential improvement and control the remaining error terms. 

The complete proof can be found in Appendix~\ref{app:proof-thm-pmd-euclidean}. 

Connecting the above policy improvement to the Nash regret (see proof in Appendix~\ref{app:proof-thm-pmd-euclidean}), we obtain the following Nash regret guarantee. 

\begin{theorem}[PMD with Euclidean Regularization] 
\label{thm:exact-pmd-euclidean}
Let Assumption~\ref{hyp:positive-discounted-visit-distrib} hold. 
Set~$\eta = \frac{1-\gamma}{4  \phi_{\max} \sum_{i=1}^N|\mathcal{A}_i|}$. Then the individual policies~$(\pi_i^t)_{i \in \mn, t \in \{1, \cdots, T\}}$ obtained from running PMD with Euclidean regularization~\eqref{eq:pmd-euclidean} and constant step size~$\eta$ for~$T \geq 1$ iterations satisfy the following,  
\begin{equation}
\text{Nash-regret(T)} \leq  12 \sqrt{ \frac{2 \phi_{\max}^2 \tilde{\kappa}_{\rho} \sum_{i=1}^N |\mathcal{A}_i| }{(1-\gamma)^4 T}}\,.
\end{equation}
Hence, the number of iterations to achieve $\epsilon$-Nash regret is: 
\begin{equation*}
T \geq \frac{288 \phi_{\max}^2 \tilde{\kappa}_{\rho} \sum_{i=1}^N |\mathcal{A}_i|}{(1-\gamma)^4 \epsilon^{2}}\,.
\end{equation*}

\end{theorem}

\noindent\textbf{Comparison to Prior Work.} We provide a few comments regarding Theorem~\ref{thm:exact-pmd-euclidean} in the light of the existing literature: 
\begin{itemize}
\item Our result improves over the convergence rates provided in \cite{ding2022independent} in terms of the distribution mismatch coefficient~$\tilde{\kappa}_{\rho}$ which can scale with the state space size~$|\ms|$ in the worst case, reducing the dependence from~$\tilde{\kappa}_{\rho}^4$ to~$\tilde{\kappa}_{\rho}$ (see Table~\ref{tab:related-work}).

Moreover, this rate matches the~$\O \left( \frac{\tilde{\kappa}_{\rho} N \underset{i \in \mn}{\max} |\mai|}{(1-\gamma)^4 \epsilon^2} \right)$ regret bound provided for the specific identical interest case in~\cite{ding2022independent} (Theorem 2) where~$\phi_{\max} = 1$ since their rewards are bounded in~$[0,1]$. This closes the gap between the purely identical interest case and the more general MPG setting. However, this result leverages the fact that the potential function is an expected discounted cumulative sum of `state-wise' potential functions. While the result in \cite{ding2022independent} applies to a more general definition of MPGs which does not require the cumulative sum form of the total potential function, we improve over this result with this additional structure. Our result for the Euclidean setting mainly serves as a comparison point for our upcoming result for PMD under KL regularization.

\item Overall, we follow a proof strategy similar to \cite{ding2022independent,zhang-et-al22softmax}. 
However, we highlight that compared to \cite{ding2022independent}, we rely on a different decomposition of the policy improvement. 
The refined decomposition we use was considered in a similar way in~\cite{cen2022independent} in the context of stateless potential games when analyzing an independent natural gradient method with entropy regularization and in~\cite{zhang-et-al22softmax} when analyzing softmax gradient play. In the present work, we address the case of MPGs beyond the particular case of (stateless) potential games and we do not consider any \textit{additional} entropy regularization. Importantly, our goal is to learn Nash equilibria for an unregularized game and our iteration complexity is of the order~$\mathcal{O}(\epsilon^{-2})$ compared to~$\mathcal{O}(\epsilon^{-4})$ for the algorithm proposed in \cite{cen2022independent}. 
\end{itemize}

Notice that the iteration complexity in Theorem~\ref{thm:exact-pmd-euclidean} scales linearly with the number of players~$N$. In the next section, we analyze the Nash regret incurred by PMD using KL regularization instead of Euclidean regularization. In this case, we show that the dependence on the number of players~$N$ can be improved to scale with~$\sqrt{N}$. The resulting Natural Policy Gradient algorithm requires a different treatment, which we undertake in the following section.

\subsection{Analysis of PMD with KL Regularization}

It has been shown in \cite{fox-et-al22indep-npg-asymp} that the iterates of independent Natural Policy Gradient with softmax policy parametrization converge to fixed point policies of the multiplicative weights algorithm (namely~\eqref{eq:pmd-KL}). 
This asymptotic result required an assumption on these fixed points. Before stating this assumption which we will also make to guarantee convergence, we recall that the fixed points of~\eqref{eq:pmd-KL} are the policies~$\pi \in \Pi$ satisfying either~$\pi_i(a_i|s) = 0$ or~$\bar{A}_i^{\pi}(s,a_i) = 0$ for every~$s \in \ms, i \in \mn, a_i \in \mai$ as can be readily seen in~\eqref{eq:pmd-KL}.  

\begin{assumption}
\label{hyp:isolated-fixed-points}
The fixed points of~\eqref{eq:pmd-KL} are isolated. 
\end{assumption}
Notice that Assumption~\ref{hyp:isolated-fixed-points} has been made in several prior works \cite{fox-et-al22indep-npg-asymp,zhang-et-al22softmax,sun-et-al23provably}. 

Similarly to the previous section, we start by quantifying the potential improvement for our algorithm. 

\begin{proposition}[Potential Improvement - KL PMD]
\label{prop:natural-grad-policy-improvement}
Under Assumption~\ref{hyp:positive-discounted-visit-distrib}, for any~$\mu \in \Delta(\mathcal{S}), t \geq 1,$ we have 
\begin{multline*} 
\Phi^{\pi^{(t+1)}}(\mu) - \Phi^{\pi^{(t)}}(\mu) \geq \left(\frac{1}{\eta} - \frac{\phi_{\max} \sqrt{N}}{(1-\gamma)^2}\right) \cdot\\ \sum_{s \in \mathcal{S}} d_{\mu}^{\pi^{(t+1)}}(s) \text{KL}(\pi^{(t+1)}_s||\, \pi^{(t)}_s)
+ \frac{1}{\eta} \sum_{s \in \mathcal{S}} d_{\mu}^{\pi^{(t+1)}}(s) \sum_{i=1}^N \log Z_{t}^{i,s}. 
\end{multline*}
Moreover, if~$\eta \leq \frac{(1-\gamma)^2}{\phi_{\max} \sqrt{N}}$, then
\begin{equation}
\label{eq:potential-improvement-KL}
\Phi^{\pi^{(t+1)}}(\mu) - \Phi^{\pi^{(t)}}(\mu) \geq 
\frac{1}{\eta} \sum_{s \in \mathcal{S}} d_{\mu}^{\pi^{(t+1)}}(s) \sum_{i=1}^N \log Z_{t}^{i,s}\,.
\end{equation}
\end{proposition}

Note that the potential improvement bound now features a dependence on~$\sqrt{N}$ instead of~$N$ in comparison to Proposition~\ref{prop:potential-improvement-euclidean}. Crucially, this improvement allows to take larger step sizes for PMD with KL regularization which leads to our improved iteration complexity. 
Using Proposition~\ref{prop:natural-grad-policy-improvement}, we connect the right-hand side of~\eqref{eq:potential-improvement-KL} to the Nash regret using Lemma~21 in \cite{zhang-et-al22softmax} which is adapted and reported in Lemma~\ref{lem:lemma21-zhang} in the appendix.
We obtain the following main result whose complete proof is deferred to Appendix~\ref{app:proof-thm-KL}. 

\begin{theorem}[PMD with KL Regularization]
\label{thm:exact-pmd-natural-gradient}
Let Assumptions~\ref{hyp:positive-discounted-visit-distrib} and~\ref{hyp:isolated-fixed-points} hold. 
Set~$\eta = \frac{1-\gamma}{2\phi_{\max} \sqrt{N}}\,.$ Then the individual policies~$(\pi_i^t)_{i \in \mn, t \in \{1, \cdots, T\}}$ obtained from running Algorithm~\eqref{eq:pmd-KL} for~$T \geq 1$ iterations with constant step size~$\eta$ satisfy the following regret guarantee, 
\begin{equation}
\text{Nash-regret}(T) \leq \sqrt{\frac{12 \phi_{\max}^2 \tilde{\kappa}_{\rho} \sqrt{N}}{ (1-\gamma)^4 c\, T}}\,,
\end{equation}
where~$c \triangleq \underset{i \in \mn}{\min}\underset{t \in \bN}{\inf}\, \underset{s \in \mathcal{S}}{\min} \underset{a_i^* \in \underset{a_i \in \mai}{\argmax}\bar{Q}^{\pi^{(t)}}_i(s,a_i)}{\sum}\,   \pi_i^{(t)}(a_i^*|s) >0 \,.$
The number of iterations to achieve~$\epsilon$-Nash regret is: 
\begin{equation*}
T \geq \frac{12 \phi_{\max}^2 \tilde{\kappa}_{\rho} \sqrt{N} }{(1-\gamma)^4 c\, \epsilon^{2}}\,.
\end{equation*}
\end{theorem}

Some comments are in order about Theorem~\ref{thm:exact-pmd-natural-gradient}: 

\begin{itemize}

\item First and foremost, observe the improved~$\sqrt{N}$ dependence on the number players in the theorem instead of~$N$ in Theorem~\ref{thm:exact-pmd-euclidean} which uses Euclidean regularization.

\item It is also worth noting that there is no dependence on the size of the action space anymore. Furthermore, the dependence on the state space size is only indirect via the minimax distribution mismatch coefficient~$\tilde{\kappa}_{\rho}$. This offers the possibility of further extension of the result to the continuous state space setting. 

\item The positivity of the constant~$c$ has been previously shown in \cite{zhang-et-al22softmax} using Assumption~\ref{hyp:isolated-fixed-points}. 
Relaxing this assumption and removing the dependence on the constant~$c$ (which can be very small) is an interesting question left for future work. 

\item The proof of this result uses similar arguments to the proof of Theorem 5 in~\cite{zhang-et-al22softmax}. Our improved result stems from a tighter policy improvement lemma than Lemma~20 in \cite{zhang-et-al22softmax}. More precisely, we control an error term differently using similar techniques to the proofs in~\cite{cen2022independent} which is concerned with stateless (i.e. static) potential games enhanced with entropy regularization. However, notice here that (a)~we consider MPGs which have stateless potential games as a particular case and (b)~we are interested in natural policy gradient (PMD with KL regularization) without the additional entropic regularization considered in~\cite{cen2022independent}.  

\item In contrast to the single-agent setting where a unified analysis of PMD for different mirror maps has been developed \cite{xiao22jmlr}, our analysis is tailored to the choice of the mirror map. Whether a unified analysis can be performed remains an open question.  
\end{itemize}

\section{Conclusion and Future Work}

In this work, we introduced an independent PMD algorithm for MPGs which subsumes as particular cases the projected Q-ascent and the celebrated Natural Policy Gradient algorithms, corresponding to PMD with Euclidean and entropic mirror maps respectively. We provided Nash regret guarantees for both algorithms showing an improved dependence on the number of agents when using Natural Policy Gradient. Notably, the Nash regret guarantee for the latter algorithm is independent of the size of the action space in contrast with the iteration complexity for PMD with Euclidean regularization. 

There are several interesting directions for future research: 
\begin{itemize}
\item Relaxing the assumption of isolated stationary policies is an interesting question for future work. Regularization might offer a way to get rid of this assumption as well as the constant~$c$ as pursued in \cite{zhang-et-al22softmax}. Achieving this while also obtaining our improved dependence on the number of players is an interesting open question.   
\item Obtaining guarantees on the last-iterate of the algorithm instead of our average Nash regret result is an interesting question to address to ensure that implementing the last-iterate policies indeed leads to Nash equilibria. 
\item Extending the result to the stochastic setting where average Q-functions are estimated from the immediate rewards observed by the agents, i.e. designing payoff-based methods is a fruitful avenue for future work. While the extension for the Euclidean case is straightforward, obtaining the result for the Natural Policy Gradient seems more delicate as it is not immediate to show that the constant~$c$ is indeed positive in the stochastic setting. 
\item While our focus in this paper was on improving the dependence on the number of agents, scaling the algorithm to large state spaces is also an important challenge. Extending our results in this direction using function approximation merits further investigation. Some results were recently discussed along these lines for PMD with Euclidean regularization \cite{ding2022independent}.  
\end{itemize}

\bibliography{refs}

\begin{thebibliography}{10}

\bibitem{aydin-eksin23networked}
S.~Ayd{\i}n and C.~Eksin.
\newblock Networked policy gradient play in markov potential games.
\newblock In {\em IEEE International Conference on Acoustics, Speech and Signal
  Processing (ICASSP)}, pages 1--5, 2023.

\bibitem{aydin-eksin23timevarying}
S.~Aydin and C.~Eksin.
\newblock Policy gradient play over time-varying networks in markov potential
  games.
\newblock In {\em 2023 62nd IEEE Conference on Decision and Control (CDC)},
  pages 1997--2002, 2023.

\bibitem{cen2022independent}
S.~Cen, F.~Chen, and Y.~Chi.
\newblock Independent natural policy gradient methods for potential games:
  Finite-time global convergence with entropy regularization.
\newblock In {\em 2022 IEEE 61st Conference on Decision and Control (CDC)},
  pages 2833--2838, 2022.

\bibitem{daskalakis2009complexity}
C.~Daskalakis, P.~W. Goldberg, and C.~H. Papadimitriou.
\newblock The complexity of computing a nash equilibrium.
\newblock {\em SIAM Journal on Computing}, 39(1):195--259, 2009.

\bibitem{ding2022independent}
D.~Ding, C.-Y. Wei, K.~Zhang, and M.~Jovanovic.
\newblock Independent {Policy} {Gradient} for {Large}-{Scale} {Markov}
  {Potential} {Games}: {Sharper} {Rates}, {Function} {Approximation}, and
  {Game}-{Agnostic} {Convergence}.
\newblock In {\em Proceedings of the 39th {International} {Conference} on
  {Machine} {Learning}}, pages 5166--5220. PMLR, June 2022.
\newblock ISSN: 2640-3498.

\bibitem{fox-et-al22indep-npg-asymp}
R.~Fox, S.~M. Mcaleer, W.~Overman, and I.~Panageas.
\newblock Independent natural policy gradient always converges in markov
  potential games.
\newblock In {\em Proceedings of The 25th International Conference on
  Artificial Intelligence and Statistics}, 2022.

\bibitem{guo-et-al23alpha-mpgs}
X.~Guo, X.~Li, C.~Maheshwari, S.~Sastry, and M.~Wu.
\newblock Markov alpha-potential games: Equilibrium approximation and regret
  analysis.
\newblock {\em arXiv preprint arXiv:2305.12553}, 2023.

\bibitem{heliou-et-al17}
A.~Heliou, J.~Cohen, and P.~Mertikopoulos.
\newblock Learning with bandit feedback in potential games.
\newblock {\em Advances in Neural Information Processing Systems}, 30, 2017.

\bibitem{jordan-barakat-he24}
P.~Jordan, A.~Barakat, and N.~He.
\newblock Independent learning in constrained {M}arkov potential games.
\newblock In {\em Proceedings of The 27th International Conference on
  Artificial Intelligence and Statistics}, 2024.

\bibitem{lan23pmd}
G.~Lan.
\newblock Policy mirror descent for reinforcement learning: Linear convergence,
  new sampling complexity, and generalized problem classes.
\newblock {\em Mathematical programming}, 198(1):1059--1106, 2023.

\bibitem{leonardos2021global}
S.~Leonardos, W.~Overman, I.~Panageas, and G.~Piliouras.
\newblock Global convergence of multi-agent policy gradient in markov potential
  games.
\newblock In {\em International Conference on Learning Representations}, 2022.

\bibitem{maheshwari-et-al22}
C.~Maheshwari, M.~Wu, D.~Pai, and S.~Sastry.
\newblock Independent and decentralized learning in markov potential games.
\newblock {\em arXiv preprint arXiv:2205.14590}, 2022.

\bibitem{marden12}
J.~R. Marden.
\newblock State based potential games.
\newblock {\em Automatica}, 48(12):3075--3088, 2012.

\bibitem{monderer-shapley96}
D.~Monderer and L.~S. Shapley.
\newblock Potential games.
\newblock {\em Games and economic behavior}, 14(1):124--143, 1996.

\bibitem{narasimha-et-al22}
D.~Narasimha, K.~Lee, D.~Kalathil, and S.~Shakkottai.
\newblock Multi-agent learning via markov potential games in marketplaces for
  distributed energy resources.
\newblock In {\em 2022 IEEE 61st Conference on Decision and Control (CDC)},
  pages 6350--6357. IEEE, 2022.

\bibitem{ozdaglar-sayin-zhang21survey}
A.~Ozdaglar, M.~O. Sayin, and K.~Zhang.
\newblock Independent learning in stochastic games.
\newblock {\em Invited chapter for the International Congress of Mathematicians
  2022 (ICM 2022)}, 2022.

\bibitem{rockafellar-et-al70cvx-analysis}
R.~T. Rockafellar.
\newblock {\em Convex Analysis}.
\newblock Princeton University Press, Princeton, 1970.

\bibitem{shapley53}
L.~S. Shapley.
\newblock Stochastic games.
\newblock {\em Proceedings of the national academy of sciences},
  39(10):1095--1100, 1953.

\bibitem{song-mei-bai22}
Z.~Song, S.~Mei, and Y.~Bai.
\newblock When can we learn general-sum markov games with a large number of
  players sample-efficiently?
\newblock In {\em International Conference on Learning Representations}, 2022.

\bibitem{sun-et-al23provably}
Y.~Sun, T.~Liu, R.~Zhou, P.~Kumar, and S.~Shahrampour.
\newblock Provably fast convergence of independent natural policy gradient for
  markov potential games.
\newblock In {\em Thirty-seventh Conference on Neural Information Processing
  Systems}, 2023.

\bibitem{xiao22jmlr}
L.~Xiao.
\newblock On the convergence rates of policy gradient methods.
\newblock {\em Journal of Machine Learning Research}, 23(282):1--36, 2022.

\bibitem{zhan-et-al23pmd-reg}
W.~Zhan, S.~Cen, B.~Huang, Y.~Chen, J.~D. Lee, and Y.~Chi.
\newblock Policy mirror descent for regularized reinforcement learning: A
  generalized framework with linear convergence.
\newblock {\em SIAM Journal on Optimization}, 33(2):1061--1091, 2023.

\bibitem{zhang-et-al22softmax}
R.~Zhang, J.~Mei, B.~Dai, D.~Schuurmans, and N.~Li.
\newblock On the global convergence rates of decentralized softmax gradient
  play in markov potential games.
\newblock {\em Advances in Neural Information Processing Systems},
  35:1923--1935, 2022.

\bibitem{zhang-et-al21grad-play}
R.~C. Zhang, Z.~Ren, and N.~Li.
\newblock Gradient play in stochastic games: Stationary points and local
  geometry.
\newblock {\em IFAC-PapersOnLine}, 55(30):73--78, 2022.
\newblock 25th International Symposium on Mathematical Theory of Networks and
  Systems MTNS 2022.

\bibitem{zhou-et-al23nmpgs}
Z.~Zhou, Z.~Chen, Y.~Lin, and A.~Wierman.
\newblock Convergence rates for localized actor-critic in networked {M}arkov
  potential games.
\newblock In {\em Proceedings of the Thirty-Ninth Conference on Uncertainty in
  Artificial Intelligence}, pages 2563--2573, 2023.

\end{thebibliography}
\bibliographystyle{abbrv}

\section{ACKNOWLEDGMENTS}
We thank the reviewers for their useful comments. This work was supported by an ETH Foundations of Data Science (ETH-FDS) postdoctoral fellowship. 


\onecolumn

\section{Proof of Theorem~\ref{thm:exact-pmd-euclidean} (Euclidean Regularization)}
\label{app:proof-thm-pmd-euclidean}

We divide the proof into two main steps: 
\begin{enumerate}

    \item The first step consists in estimating the policy improvement in terms of the potential function, i.e., lower bounding the quantity~$\Phi^{\pi^{(t+1)}}(\mu) - \Phi^{\pi^{(t)}}(\mu)$ where~$\pi^{(t+1)}$ is the joint policy of the agents at time~$t+1$ after one step of PMD from the joint policy~$\pi^{(t)}.$ 

    \item In the second step, we relate the Nash regret to the policy improvement controlled in the first step and conclude.
\end{enumerate}

Let~$\mu$ be any initial distribution over the state space~$\mathcal{S}\,.$ 
Recall the definition of the total potential function~$\Phi: \Pi \to \mathbb{R}$: 
\begin{equation}
\Phi^{\pi}(\mu) \triangleq \mathbb{E}_{\mu, \pi}\left[\sum_{t=0}^{\infty} \gamma^t \phi(s_t,a_t)\right]\,.
\end{equation}

Define also the $Q$-function induced by the potential function~$\phi$ and a given policy~$\pi$ over the joint state and action space for every state-action pair~$(s,a) \in \mathcal{S} \times \mathcal{A}$ as follows: 
\begin{equation}
Q_{\phi}^{\pi}(s,a) \triangleq \mathbb{E}_{\pi}\left[\sum_{t=0}^{\infty} \gamma^t \phi(s_t,a_t) | s_0 = s, a_0 = a\right]\,.
\end{equation}

\begin{proposition}
\label{lem:policy-improv-potential-fun}
For any initial distribution~$\mu \in \Delta(\mathcal{S})$, 
\begin{equation*}
\Phi^{\pi^{(t+1)}}(\mu) - \Phi^{\pi^{(t)}}(\mu) \geq \left( \frac{1}{2\eta(1-\gamma)} - \frac{2 \phi_{\max} \sum_{i=1}^N |\mathcal{A}_i|}{(1-\gamma)^3} \right) \sum_{s \in \mathcal{S}} d_{\mu}^{\pi^{(t+1)}}(s) \sum_{i=1}^N \|\pi_i^{(t+1)}(\cdot|s) - \pi_i^{(t)}(\cdot|s)\|^2\,.
\end{equation*}
\end{proposition}

\begin{proof}
Using the performance difference lemma (\cref{lem:performance-difference-single-agent}), we can first write
\begin{equation}
\label{eq:perf-diff-total-pot-fun}
\Phi^{\pi^{(t+1)}}(\mu) - \Phi^{\pi^{(t)}}(\mu) = \frac{1}{1-\gamma} \sum_{s \in \mathcal{S}, a \in \mathcal{A}} d_{\mu}^{\pi^{(t+1)}}(s) (\pi^{(t+1)}(a|s) - \pi^{(t)}(a|s))Q_{\phi}^{\pi^{(t)}}(s,a)\,.
\end{equation}

\begin{remark}
Notice here that our first step consists in using the performance difference lemma between the \textit{joint} policies~$\pi^{(t)}$ and~$\pi^{(t+1)}$ rather than the multi-agent form of the performance difference lemma for each agent~$i$ where the policy of all the agents but the $i$-th one is fixed to~$\pi_{-i}^{(t)}$ (see Lemma 1 in~\cite{ding2022independent}). 
\end{remark}

We then use the following decomposition of the policy increment: 
\begin{equation}
\label{eq:policy-increment}
\pi^{(t+1)}(a|s) - \pi^{(t)}(a|s) = \sum_{i=1}^N \left( \prod_{j=1}^{i} \pi_j^{(t+1)}(a_j|s) \prod_{j=i+1}^N \pi_j^{(t)}(a_j|s) - \prod_{j=1}^{i-1} \pi_j^{(t+1)}(a_j|s) \prod_{j=i}^N \pi_j^{(t)}(a_j|s)\right)\,,
\end{equation}
which can be verified by noticing that the policy~$\pi^{(t)}(a|s)$ has the product form: $\pi^{(t)}(a|s) = \prod_{i=1}^N \pi_i^{(t)}(a_i|s)$ for every integer~$t$, where~$a = (a_i)_{1 \leq i \leq N}$ and standard telescoping. Notice that we use the convention~$\prod_{j=1}^{0} z_j = 1$ and~$\prod_{j=N+1}^N z_j = 1$ for any sequence of nonnegative reals~$(z_j)_{j\geq 0}\,.$ 

Before proceeding with this decomposition, we introduce a few additional useful notations for every integer~$t \geq 1$, $i \in \mn, s \in \mathcal{S}$ and~$a = (a_i)_{i \in \mn} \in \mathcal{A},$
\begin{align}
\tilde{\pi}_{-i}^{(t)}(a_{-i}|s) &\triangleq \prod_{j=1}^{i-1} \pi_j^{(t+1)}(a_j|s) \prod_{j=i+1}^N \pi_j^{(t)}(a_j|s)\,, \label{eq:def-tilde-pi}\\
\tilde{Q}_{\phi,i}^{\pi^{(t)}}(s, a_i) &\triangleq \sum_{a_{-i} \in \mathcal{A}_{-i}} \tilde{\pi}_{-i}^{(t)}(a_{-i}|s)\, Q_{\phi}^{\pi^{(t)}}(s,a)\,, \label{eq:def-tilde-Q-phi}\\
\bar{Q}_{\phi,i}^{\pi^{(t)}}(s, a_i) &\triangleq \sum_{a_{-i} \in \mathcal{A}_{-i}} \pi_{-i}^{(t)}(a_{-i}|s)\, Q_{\phi}^{\pi^{(t)}}(s,a)\label{eq:def-bar-Q-phi}\,. 
\end{align}
In the rest of the proof, we will use the shorthand notations~$\tilde{Q}_{\phi,i}^{(t)}(s, a_i) \triangleq \tilde{Q}_{\phi,i}^{\pi^{(t)}}(s, a_i), Q_{\phi}^{(t)}(s,a) \triangleq Q_{\phi}^{\pi^{(t)}}(s,a), \bar{Q}_i^{(t)}(s,a_i) \triangleq \bar{Q}_i^{\pi^{(t)}}(s,a_i), d_{\mu}^{(t+1)} \triangleq d_{\mu}^{\pi^{(t+1)}}\,.$

Plugging~\eqref{eq:policy-increment} into~\eqref{eq:perf-diff-total-pot-fun}, we obtain
\begin{align}
\label{eq:Phi-t+1-Phi-t}
\Phi^{\pi^{(t+1)}}(\mu) - \Phi^{\pi^{(t)}}(\mu) 
&= \frac{1}{1-\gamma} \sum_{s \in \mathcal{S}} d_{\mu}^{(t+1)}(s) \sum_{i=1}^N \sum_{a_i \in \mathcal{A}_i} (\pi_i^{(t+1)}(a_i|s) - \pi_i^{(t)}(a_i|s)) \tilde{Q}_{\phi,i}^{(t)}(s, a_i) \nonumber\\
&= \underbrace{
\frac{1}{1-\gamma} \sum_{s \in \mathcal{S}} d_{\mu}^{(t+1)}(s) \sum_{i=1}^N \sum_{a_i \in \mathcal{A}_i} (\pi_i^{(t+1)}(a_i|s) - \pi_i^{(t)}(a_i|s)) \bar{Q}_{i}^{(t)}(s, a_i)}_{\text{Term A}} \nonumber\\
&+ \underbrace{
\frac{1}{1-\gamma} \sum_{s \in \mathcal{S}} d_{\mu}^{(t+1)}(s) \sum_{i=1}^N \sum_{a_i \in \mathcal{A}_i} (\pi_i^{(t+1)}(a_i|s) - \pi_i^{(t)}(a_i|s)) (\tilde{Q}_{\phi,i}^{(t)}(s, a_i) - \bar{Q}_{i}^{(t)}(s, a_i))}_{\text{Term B}}\,.
\end{align}

We now control each one of the terms above successively. 

\begin{enumerate}
    \item \noindent\textbf{Term A.} This term brings policy improvement as a consequence of the update rule~\eqref{eq:pmd-euclidean}. This step is the same as in the proof of Theorem~1 in \cite{ding2022independent}. The optimality condition for~$\pi^{(t+1)}$ in the update rule~\eqref{eq:pmd-euclidean} yields: 
    \begin{equation}
        \ps{\pi_i^{(t+1)}(\cdot|s) - \pi_i^{(t)}(\cdot|s), \bar{Q}_i^{\pi^{(t)}}(s,\cdot)}_{\mathcal{A}_i} \geq \frac{1}{2\eta} \|\pi_i^{t+1}(\cdot|s) - \pi_i^{(t)}(\cdot|s))\|^2\,.
    \end{equation}

As a consequence, we have
\begin{equation}
\label{eq:termA-bound}
\text{Term A} \geq \frac{1}{2 \eta (1-\gamma)} \sum_{s \in \mathcal{S}} d_{\mu}^{(t+1)}(s) \sum_{i=1}^N \|\pi_i^{(t+1)}(\cdot|s) - \pi_i^{(t)}(\cdot|s)\|^2\,.
\end{equation}

    \item \noindent\textbf{Term B.} For this term we start by observing that for every joint policy~$\pi \in \Pi$, every state~$s \in \mathcal{S}$ and every action~$a_i \in \mathcal{A}_i\,, i \in \mn\,,$
    \begin{equation}
            \label{eq:bar-Q-i-equals-bar-Q-phi}
            \bar{Q}_i^{\pi}(s,a_i) = \bar{Q}_{\phi,i}^{\pi}(s,a_i)\,.
    \end{equation}
    This is a key observation for our proof which follows from the properties of a Markov Potential Game. Indeed, first, it follows from the MPG definition that for every joint policy~$\pi$ and every~$s \in \mathcal{S}, i \in \mn, a_i \in \mai$, we have
    \begin{equation}
    \label{eq:partial-V-partial-Phi}
        \frac{\partial V_i^{\pi}(\mu)}{\partial \pi_i(a_i|s)} = \frac{\partial \Phi^{\pi}(\mu)}{\partial \pi_i(a_i|s)}\,.
    \end{equation}
    Then the policy gradient theorem \cite[Eq. (18)]{xiao22jmlr} yields an expression for each one of the quantities in the above identity:
    \begin{equation}
    \label{eq:pg-V-pg-Phi}
        \frac{\partial V_i^{\pi}(\mu)}{\partial \pi_i(a_i|s)} = \frac{1}{1-\gamma} d_{\mu}^{\pi}(s) \bar{Q}_i^{\pi}(s,a_i)\,; \quad \frac{\partial \Phi^{\pi}(\mu)}{\partial \pi_i(a_i|s)} = \frac{1}{1-\gamma} d_{\mu}^{\pi}(s) \bar{Q}_{\phi,i}^{\pi}(s,a_i)\,.
    \end{equation}
    As a consequence, we obtain~\eqref{eq:bar-Q-i-equals-bar-Q-phi} under the assumption that~$d_{\mu}^{\pi}(s) > 0$ for every joint policy~$\pi$ and every~$s \in \mathcal{S}\,.$
    We now turn to controlling the term B. For this purpose, we first control the difference between the two Q-functions in term B as follows: 
    \begin{align}
    \label{eq:bound-tilde-Q-bar-Q}
            |\tilde{Q}_{\phi,i}^{(t)}(s, a_i) - \bar{Q}_{i}^{(t)}(s, a_i)| 
            & \stackrel{(a)}{=} |\tilde{Q}_{\phi,i}^{(t)}(s, a_i) - \bar{Q}_{\phi,i}^{(t)}(s, a_i)| \nonumber\\
            & \stackrel{(b)}{=} \left| \sum_{a_{-i} \in \mathcal{A}_{-i}} (\tilde{\pi}_{-i}^{(t)}(a_{-i}|s) - \pi_{-i}^{(t)}(a_{-i}|s))\, Q_{\phi}^{\pi^{(t)}}(s,a) \right|\nonumber\\
            &\stackrel{(c)}{=} \left| \sum_{a_{-i} \in \mathcal{A}_{-i}} \left( \prod_{j=1}^{i-1} \pi_j^{(t+1)}(a_j|s) -  \prod_{j=1}^{i-1} \pi_j^{(t)}(a_j|s) \right)\, \prod_{j=i+1}^N \pi_j^{(t)}(a_j|s) Q_{\phi}^{\pi^{(t)}}(s,a) \right|\nonumber\\
            &\stackrel{(d)}{\leq} \frac{\phi_{\max}}{1-\gamma} \sum_{j=1}^N \|\pi_j^{(t+1)}(\cdot|s) - \pi_j^{(t)}(\cdot|s)\|_1\,,
    \end{align}
    where~$(a)$ follows from~\eqref{eq:bar-Q-i-equals-bar-Q-phi}, $(b)$ stems from the definitions~\eqref{eq:def-tilde-Q-phi}-\eqref{eq:def-bar-Q-phi}, $(c)$ from~\eqref{eq:def-tilde-pi} and~$(d)$ can be proved by decomposing the difference of products into a sum of differences to obtain the sum of the differences between individual policies at the successive times~$t$ and~$t+1$. 
    Using the bound~\eqref{eq:bound-tilde-Q-bar-Q}, we immediately obtain 
    \begin{align}
    \label{eq:termB-bound}
            |\text{Term B}| 
            &\leq \frac{\phi_{\max}}{(1-\gamma)^2} \sum_{s \in \mathcal{S}} d_{\mu}^{(t+1)}(s) \left( \sum_{i=1}^N \|\pi_i^{(t+1)}(\cdot|s) - \pi_i^{(t)}(\cdot|s)\|_1 \right)^2\nonumber\\
            &\leq  \frac{\phi_{\max}}{(1-\gamma)^2} \sum_{s \in \mathcal{S}} d_{\mu}^{(t+1)}(s) \left( \sum_{i=1}^N \sqrt{|\mathcal{A}_i|} \cdot \|\pi_i^{(t+1)}(\cdot|s) - \pi_i^{(t)}(\cdot|s)\| \right)^2 \nonumber\\
            &\leq \frac{\phi_{\max} \sum_{i=1}^N |\mathcal{A}_i|}{(1-\gamma)^2} \sum_{s \in \mathcal{S}} d_{\mu}^{(t+1)}(s)\sum_{i=1}^N \|\pi_i^{(t+1)}(\cdot|s) - \pi_i^{(t)}(\cdot|s)\|^2\,, 
    \end{align}
    where the last inequality follows from the Cauchy-Schwartz inequality. 
    
\end{enumerate}

Combining~\eqref{eq:termA-bound} and~\eqref{eq:termB-bound} in~\eqref{eq:Phi-t+1-Phi-t}, we obtain 
\begin{equation*}
\Phi^{\pi^{(t+1)}}(\mu) - \Phi^{\pi^{(t)}}(\mu) \geq \left( \frac{1}{2\eta(1-\gamma)} - \frac{\phi_{\max} \,\sum_{i=1}^N |\mathcal{A}_i|}{(1-\gamma)^2} \right) \sum_{s \in \mathcal{S}} d_{\mu}^{(t+1)}(s) \sum_{i=1}^N \|\pi_i^{(t+1)}(\cdot|s) - \pi_i^{(t)}(\cdot|s)\|^2\,.
\end{equation*}
This concludes the proof of the lemma.
\end{proof}

The rest of the proof of the theorem follows the same lines as the proof of Theorem 1 in \cite{ding2022independent} upon noticing that we use a sharper result for policy improvement, namely Lemma~\ref{lem:policy-improv-potential-fun}. We provide a proof in the following for completeness. 

We now prove a technical lemma that will help connecting the Nash regret with the policy improvement that we quantified in Lemma~\ref{lem:policy-improv-potential-fun}. 
\begin{lemma}
\label{lem:technical-for-regret}
For every~$i \in \mn, \pi_i' \in \Pi_i,$ and every integer~$t \geq 1$, we have for every~$s \in \mathcal{S}$,
\begin{equation}
\ps{\pi_i'(\cdot|s) - \pi_i^{(t)}(\cdot|s), \bar{Q}_i^{(t)}(s,\cdot)}_{\mathcal{A}_i} \leq \left( \frac{2}{\eta} + \frac{\sqrt{\max_{i \in \mn} |\mathcal{A}_i|}}{1-\gamma} \right) \|\pi_i^{(t+1)}(\cdot|s) - \pi_i^{(t)}(\cdot|s)\|\,.
\end{equation}
Moreover, if~$\eta \leq \frac{1-\gamma}{\sqrt{\max_{i \in \mn} |\mathcal{A}_i|}}$, then
\begin{equation}
\ps{\pi_i'(\cdot|s) - \pi_i^{(t)}(\cdot|s), \bar{Q}_i^{(t)}(s,\cdot)}_{\mathcal{A}_i} 
\leq \frac{3}{\eta} \|\pi_i^{(t+1)}(\cdot|s) - \pi_i^{(t)}(\cdot|s)\|\,.
\end{equation}
\end{lemma}

\begin{proof} 
Let~$t \geq 1, i \in \mn, \pi_i' \in \Pi_i\,.$ Then, we have
\begin{align}
\ps{\pi_i'(\cdot|s) - \pi_i^{(t)}(\cdot|s), \bar{Q}_i^{(t)}(s,\cdot)} 
&=  \ps{\pi_i'(\cdot|s) - \pi_i^{(t+1)}(\cdot|s), \bar{Q}_i^{(t)}(s,\cdot)} 
+ \ps{\pi_i^{(t+1)}(\cdot|s) - \pi_i^{(t)}(\cdot|s), \bar{Q}_i^{(t)}(s,\cdot)}\nonumber\\
&\leq \frac{1}{\eta} \ps{\pi_i'(\cdot|s) - \pi_i^{(t+1)}(\cdot|s), \pi_i^{(t+1)}(\cdot|s) - \pi_i^{(t)}(\cdot|s)}\nonumber\\
&+ \ps{\pi_i^{(t+1)}(\cdot|s) - \pi_i^{(t)}(\cdot|s), \bar{Q}_i^{(t)}(s,\cdot)}\,,
\end{align}
where the last step follows from using the first-order optimality condition for the optimization problem defining the update rule~\eqref{eq:pmd-euclidean}. Then we use Cauchy-Schwarz inequality and notice that~$\|\pi_i'(\cdot|s) - \pi_i^{(t+1)}(\cdot|s)\| \leq \|\pi_i'(\cdot|s) - \pi_i^{(t+1)}(\cdot|s)\|_1 \leq 2$ and~$\|\bar{Q}_i^{(t)}(s,\cdot)\| \leq \sqrt{\max_{i \in \mn} |\mathcal{A}_i|}/(1-\gamma)$ to conclude.
\end{proof}

We are now ready to bound the Nash regret. We start with the following series of inequalities, using~$i$ for the index of the maximum over~$\mn$, 

\begin{align}
&\sum_{t=1}^T \max_{i \in \mn} \max_{\pi_i'} V_i^{\pi_i', \pi_{-i}^{(t)}}(\rho) - V_i^{\pi^{(t)}}(\rho)\nonumber\\ 
&\stackrel{(a)}{=} \frac{1}{1-\gamma} \sum_{t=1}^T  \max_{\pi_i'} \sum_{s \in \mathcal{S}, a_i \in \mathcal{A}_i} d_{\rho}^{\pi_i', \pi_{-i}^{(t)}}(s) (\pi_i'(a_i|s) - \pi_i^{(t)}(a_i|s)) \bar{Q}_i^{(t)}(s,a_i)\nonumber\\
&\stackrel{(b)}{\leq} \frac{3}{\eta(1-\gamma)} \sum_{t=1}^T \sum_{s \in \mathcal{S}}  d_{\rho}^{\pi_i', \pi_{-i}^{(t)}}(s)  \|\pi_i^{(t+1)}(\cdot|s) - \pi_i^{(t)}(\cdot|s)\|\nonumber\\
&= \frac{3}{\eta(1-\gamma)} \sum_{t=1}^T \sum_{s \in \mathcal{S}} \sqrt{\frac{d_{\rho}^{\pi_i', \pi_{-i}^{(t)}}(s)}{d_{\nu}^{\pi^{(t+1)}}(s)}} \sqrt{d_{\nu}^{\pi^{(t+1)}}(s)} \sqrt{d_{\rho}^{\pi_i', \pi_{-i}^{(t)}}(s)}  \|\pi_i^{(t+1)}(\cdot|s) - \pi_i^{(t)}(\cdot|s)\| \nonumber\\
&\stackrel{(c)}{\leq} \frac{3 \sqrt{\sup_{\pi \in \Pi} \left\| \frac{d_{\rho}^{\pi}}{\nu}  \right\|_{\infty}}}{\eta (1-\gamma)^{3/2}} \sqrt{\sum_{t=1}^T \sum_{s \in \mathcal{S}} d_{\rho}^{\pi_i', \pi_{-i}^{(t)}}(s)} \sqrt{\sum_{t=1}^T \sum_{s \in \mathcal{S}} d_{\nu}^{\pi^{(t+1)}}(s)  \sum_{j=1}^N \|\pi_j^{(t+1)}(\cdot|s) - \pi_j^{(t)}(\cdot|s)\|^2}\nonumber\\
&\stackrel{(d)}{\leq} \frac{3 \sqrt{\sup_{\pi \in \Pi} \left\| \frac{d_{\rho}^{\pi}}{\nu}  \right\|_{\infty} T}}{\eta (1-\gamma)^{3/2}} \sqrt{4 \eta (1-\gamma) (\Phi^{\pi^{(T+1)}}(\nu) - \Phi^{\pi^{(1)}}(\nu))}\nonumber\\
&\stackrel{(e)}{\leq}  6 \sqrt{ \frac{ 2\sup_{\pi \in \Pi} \left\| \frac{d_{\rho}^{\pi}}{\nu}  \right\|_{\infty} \phi_{\max} T}{\eta (1-\gamma)^3}}\nonumber\\
&\stackrel{(f)}{\leq} 6 \sqrt{ \frac{2 \tilde{\kappa}_{\rho} \phi_{\max} T}{\eta (1-\gamma)^3}}\,,
\end{align}
where~$(a)$~follows from the multi-agent performance difference lemma, $(b)$~uses Lemma~\ref{lem:technical-for-regret}, $(c)$~holds for any distribution~$\nu$ with full support on the state space~$\mathcal{S}$, uses the inequality~$d_{\nu}^{\pi^{(t+1)}}(s) \geq (1-\gamma) \nu(s)$ together with the Cauchy-Schwarz inequality and upperbounds the $i$-th term by the sum over the $N$~players, $(d)$~uses Lemma~\ref{lem:policy-improv-potential-fun}, $(e)$~follows from the boundedness of the potential function~$\phi$ and finally~$(f)$ stems from the definition of the constant~$\tilde{\kappa}_{\rho}\,.$

Finally, setting~$\eta = \frac{1-\gamma}{4 (\phi_{\max} - \phi_{\min}) \sum_{i=1}^N |\mathcal{A}_i|}$, we obtain: 
\begin{equation}
\text{Nash-regret(T)} \leq  12 \sqrt{ \frac{2 \phi_{\max}^2 \tilde{\kappa}_{\rho} \sum_{i=1}^N |\mathcal{A}_i|}{(1-\gamma)^4 T}}\,.
\end{equation}
Hence, the iteration complexity to achieve $\epsilon$-Nash regret is up to numerical constants: 
\begin{equation*}
\mathcal{O}\left( \frac{\phi_{\max}^2 \tilde{\kappa}_{\rho} \sum_{i=1}^N |\mathcal{A}_i|}{(1-\gamma)^4 \epsilon^{2}} \right)\,.
\end{equation*}

\begin{remark}
Notice that
\begin{equation}
\tilde{\kappa}_{\rho} \leq \min(\kappa_{\rho}, |\ms|)\,, 
\end{equation}
where~$\kappa_{\rho} \triangleq \sup_{\pi \in \Pi} \left\| \frac{d_{\rho}^{\pi}}{\rho} \right\|_{\infty}\,.$
Observe for this that setting~$\nu$ to be the uniform distribution over the state space~$\mathcal{S}$ gives~$\tilde{\kappa}_{\rho} \leq \sup_{\pi \in \Pi} \left\| \frac{d_{\rho}^{\pi}}{\nu}  \right\|_{\infty} \leq |\ms|$ whereas~$\nu = \rho$ yields~$\tilde{\kappa}_{\rho} \leq \kappa_{\rho}\,.$
\end{remark}

\newpage
\section{Proof of Theorem~\ref{thm:exact-pmd-natural-gradient} (KL Regularization)}
\label{app:proof-thm-KL}

Recall first from~\eqref{eq:pmd-KL} that the update rule of our PMD algorithm with KL regularization can also be written as follows: 
\begin{equation}
\label{eq:update-rule-pmd-KL-adv}
\pi_i^{(t+1)}(a_i|s) =  \frac{\pi_i^{(t)}(a_i|s) \exp( \frac{\eta}{1-\gamma} \bar{A}_{i}^{\pi^{(t)}}(s,a_i))}{Z_t^{i,s}}\,, 
\end{equation}
with~$Z_t^{i,s}\triangleq \sum_{a_i \in \mai} \pi_i^{(t)}(a_i|s) \exp(\frac{\eta}{1-\gamma}\, \bar{A}_{i}^{\pi^{(t)}}(s,a_i))\,.$

\begin{remark}
Notice that we introduced a~$(1-\gamma)$ factor here for convenience. With a slight abuse of notation, we also use in the rest of this section~$\eta$ as the stepsize, which is up to the constant~$1-\gamma$ the same as the stepsize~$\eta$ in~\eqref{eq:pmd-KL}. Note that we also use with a slight abuse of notation the same~$Z_t^{i,s}$ notation even when using the advantage function. 
\end{remark}

The first lemma establishes a policy improvement result for the natural policy gradient algorithm for MPGs. Crucially, compared to Lemma 20 in~\cite{zhang-et-al22softmax}, our result allows to take a larger step size. In particular, the dependence of the step size on the number of agents is of the order of~$\sqrt{N}$ instead of~$N$. 
\begin{proposition}
\label{lem:natural-grad-policy-improvement}
Let Assumption~\ref{hyp:positive-discounted-visit-distrib} hold. 
For any initial distribution~$\mu \in \Delta(\ms)$, the iterates of PMD with KL regularization~\eqref{eq:pmd-KL} satisfy for every time step~$t \geq 1$, 
\begin{multline} 
\label{eq:natural-grad-policy-improvement}
\Phi^{\pi^{(t+1)}}(\mu) - \Phi^{\pi^{(t)}}(\mu) \geq \left(\frac{1}{\eta} - \frac{\phi_{\max} \sqrt{N}}{(1-\gamma)^2}\right) \sum_{s \in \mathcal{S}} d_{\mu}^{\pi^{(t+1)}}(s) \text{KL}(\pi^{(t+1)}_s\,||\, \pi^{(t)}_s)
+ \frac{1}{\eta} \sum_{s \in \mathcal{S}} d_{\mu}^{\pi^{(t+1)}}(s) \sum_{i=1}^N \log Z_{t}^{i,s}\,,
\end{multline}
where we recall that~$Z_t^{i,s} = \sum_{a_i \in \mathcal{A}_i} \pi_i^{(t)}(a_i|s) \exp\left(\frac{\eta \bar{A}_i^{(t)}(s,a_i)}{1-\gamma}\right)$ for every~$i \in \mn, s \in \mathcal{S}$ and every integer~$t \geq 1\,.$
Moreover, if~$\eta \leq \frac{(1-\gamma)^2}{\phi_{\max} \sqrt{N}}$, then
\begin{equation}
\Phi^{\pi^{(t+1)}}(\mu) - \Phi^{\pi^{(t)}}(\mu) \geq 
\frac{1}{\eta} \sum_{s \in \mathcal{S}} d_{\mu}^{\pi^{(t+1)}}(s) \sum_{i=1}^N \log Z_{t}^{i,s}\,.
\end{equation}
\end{proposition}

\begin{proof}
We first introduce a few additional notations for every~$s \in \mathcal{S}, a = (a_i)_{i \in \mn} \in \mathcal{A}$ and~$i \in \mn$: 
\begin{align}
V_{\phi}^{\pi^{(t)}}(s) &\triangleq \mathbb{E}_{\mu, \pi}\left[\sum_{t=0}^{\infty} \gamma^t \phi(s_t,a_t) | s_0 = s\right]\,,\\
A_{\phi}^{\pi^{(t)}}(s,a) &\triangleq Q_{\phi}^{\pi^{(t)}}(s,a) - V_{\phi}^{\pi^{(t)}}(s)\,, \label{eq:advantage-fun-potential}\\
\tilde{A}_{\phi,i}^{\pi^{(t)}}(s, a_i) &\triangleq \sum_{a_{-i} \in \mathcal{A}_{-i}} \tilde{\pi}_{-i}^{(t)}(a_{-i}|s)\, A_{\phi}^{\pi^{(t)}}(s,a)\,, \label{eq:def-tilde-A-phi}\\
\bar{A}_{\phi,i}^{\pi^{(t)}}(s, a_i) &\triangleq \sum_{a_{-i} \in \mathcal{A}_{-i}} \pi_{-i}^{(t)}(a_{-i}|s)\, A_{\phi}^{\pi^{(t)}}(s,a)\label{eq:def-bar-A-phi}\,. 
\end{align}
We will also use again the shorthand notation~$d_{\mu}^{(t+1)} = d_{\mu}^{\pi^{(t+1)}}\,.$\\

Similarly to the proof of Lemma~\ref{lem:policy-improv-potential-fun} (and Lemma~20 in~\cite{zhang-et-al22softmax}), the performance difference lemma yields: 
\begin{align}
\Phi^{\pi^{(t+1)}}(\mu) - \Phi^{\pi^{(t)}}(\mu) 
&= \frac{1}{1-\gamma} \sum_{s \in \mathcal{S}} d_{\mu}^{(t+1)}(s) \sum_{i=1}^N \sum_{a_i \in \mathcal{A}_i} (\pi_i^{(t+1)}(a_i|s) - \pi_i^{(t)}(a_i|s)) \tilde{A}_{\phi,i}^{(t)}(s, a_i) \nonumber\\
&= \underbrace{
\frac{1}{1-\gamma} \sum_{s \in \mathcal{S}} d_{\mu}^{(t+1)}(s) \sum_{i=1}^N \sum_{a_i \in \mathcal{A}_i} (\pi_i^{(t+1)}(a_i|s) - \pi_i^{(t)}(a_i|s)) \bar{A}_{i}^{(t)}(s, a_i)}_{\text{Term A}} \nonumber\\
&+ \underbrace{
\frac{1}{1-\gamma} \sum_{s \in \mathcal{S}} d_{\mu}^{(t+1)}(s) \sum_{i=1}^N \sum_{a_i \in \mathcal{A}_i} (\pi_i^{(t+1)}(a_i|s) - \pi_i^{(t)}(a_i|s)) (\tilde{A}_{\phi,i}^{(t)}(s, a_i) - \bar{A}_{i}^{(t)}(s, a_i))}_{\text{Term B}}\,,
\end{align}
where we used the shorthand notations~$\tilde{A}_{\phi,i}^{(t)}(s, a_i) = \tilde{A}_{\phi,i}^{\pi^{(t)}}(s, a_i), \bar{A}_{i}^{(t)}(s, a_i) = \bar{A}_{i}^{\pi^{(t)}}(s, a_i)\,.$

We now control each one of the terms separately in what follows. 

\begin{enumerate}
\item \textbf{Term A.}
For this term, our treatment is the same as in \cite{zhang-et-al22softmax}. We provide a proof here for completeness and we also add a crucial precision in the proof that seemed to be used in~\cite{zhang-et-al22softmax} as a claim.  This is the fact that~$\bar{A}_{i}^{\pi}(s,a_i) = \bar{A}_{\phi,i}^{\pi}(s,a_i)$ for every~$\pi \in \Pi, s \in \ms, i \in \mn, a_i \in \mai$, i.e. the averaged advantage function~$\bar{A}_{i}^{\pi}$ coincides with the averaged advantage function induced by the potential function~$\bar{A}_{\phi,i}^{\pi}$ (see~\eqref{eq:advantage-fun-potential}-\eqref{eq:def-bar-A-phi} for a definition). 
This result is a consequence of the equality between the partial derivatives of the individual value functions and the partial derivative of the potential function w.r.t. the individual policies entries, as implied by the definition of MPGs (see \eqref{eq:partial-V-partial-Phi} and~\eqref{eq:pg-V-pg-Phi}). Before proceeding, we provide a brief proof of this fact that will be useful later on. Since~$\bar{Q}_i^{\pi}(s,a_i) = \bar{Q}_{\phi}^{\pi}(s,a_i)$ (see \eqref{eq:bar-Q-i-equals-bar-Q-phi}) and~$V_{\phi}^{\pi}(s) = \sum_{a_i' \in \mai} \pi_i(a_i'|s) \bar{Q}_{\phi}^{\pi}(s,a_i')$, we also have the desired equality~$\bar{A}_{i}^{\pi}(s,a_i) = \bar{A}_{\phi,i}^{\pi}(s,a_i)$ using the definition of the averaged advantage function.\\

It follows from our update rule~\eqref{eq:update-rule-pmd-KL-adv} that: 
\begin{equation}
\label{eq:advantage-update-as-fun-update-rule}
\bar{A}_{i}^{(t)}(s,a_i) = \frac{1-\gamma}{\eta} \left( \log\left( \frac{\pi_i^{(t+1)}(a_i|s)}{\pi_i^{(t)}(a_i|s)}\right) + \log Z_t^{i,s} \right)\,.
\end{equation}

Plugging~\eqref{eq:advantage-update-as-fun-update-rule} into the Term A immediately implies that 
\begin{align}
\label{eq:bound-termA-natural-grad}
\text{Term A} 
&= \frac{1}{1-\gamma} \sum_{s \in \mathcal{S}} d_{\mu}^{(t+1)}(s) \sum_{i=1}^N \sum_{a_i \in \mathcal{A}_i} (\pi_i^{(t+1)}(a_i|s) - \pi_i^{(t)}(a_i|s)) \bar{A}_{i}^{(t)}(s, a_i) \nonumber\\
&\stackrel{(a)}{=} \frac{1}{1-\gamma} \sum_{s \in \mathcal{S}} d_{\mu}^{(t+1)}(s) \sum_{i=1}^N \sum_{a_i \in \mathcal{A}_i} \pi_i^{(t+1)}(a_i|s) \bar{A}_{i}^{(t)}(s, a_i) \nonumber\\
&= \frac{1}{\eta} \sum_{s \in \mathcal{S}} d_{\mu}^{(t+1)}(s)\sum_{i=1}^N \sum_{a_i \in \mathcal{A}_i} \pi_i^{(t+1)}(a_i|s) \left( \log\left( \frac{\pi_i^{(t+1)}(a_i|s)}{\pi_i^{(t)}(a_i|s)}\right) + \log Z_t^{i,s} \right) \nonumber\\
&= \frac{1}{\eta} \sum_{s \in \mathcal{S}} d_{\mu}^{(t+1)}(s) \sum_{i=1}^N \text{KL}(\pi_i^{(t+1)}(\cdot|s)\,||\,\pi_i^{(t)}(\cdot|s)) 
+ \frac{1}{\eta} \sum_{s \in \mathcal{S}} d_{\mu}^{(t+1)}(s) \sum_{i=1}^N \log Z_t^{i,s}\nonumber\\
&= \frac{1}{\eta} \sum_{s \in \mathcal{S}} d_{\mu}^{(t+1)}(s) \text{KL}(\pi^{(t+1)}(\cdot|s)\,||\,\pi^{(t)}(\cdot|s)) 
+ \frac{1}{\eta} \sum_{s \in \mathcal{S}} d_{\mu}^{(t+1)}(s) \sum_{i=1}^N \log Z_t^{i,s} \,,
\end{align}
where~$(a)$ follows from observing that~$\sum_{a_i \in \mathcal{A}_i}\pi_i^{(t)}(a_i|s) \bar{A}_{i}^{(t)}(s, a_i) = 0$ and the last identity follows from the additivity of the KL divergence for product distributions. 

\item \textbf{Term B.} Our improvement w.r.t. Theorem~5 in~\cite{zhang-et-al22softmax} comes from the way we control this term. The following derivations are similar to parts of the proofs in \cite{cen2022independent}, namely Appendix~A and the end of Appendix~B.1. Recall though that the work \cite{cen2022independent} which inspired our analysis only deals with (stateless) potential games and considers a different natural policy gradient algorithm enhanced with entropy regularization.  
First of all, we have that
\begin{align}
|\tilde{A}_{\phi,i}^{(t)}(s,a_i) - \bar{A}_i^{(t)}(s,a_i)| 
&\stackrel{(a)}{=} |\tilde{A}_{\phi,i}^{(t)}(s,a_i) - \bar{A}_{\phi,i}^{(t)}(s,a_i)| \nonumber\\
&\stackrel{(b)}{=} \left|\bE_{a_{-i} \sim \tilde{\pi}_{-i}^{(t)}}[A_{\phi}^{(t)}(s,a_i,a_{-i})] - \bE_{a_{-i} \sim \pi_{-i}^{(t)}}[A_{\phi}^{(t)}(s,a_i,a_{-i})] \right| \nonumber\\
&\stackrel{(c)}{=} \left|\bE_{a_{-i} \sim \tilde{\pi}_{-i}^{(t)}}[Q_{\phi}^{(t)}(s,a_i,a_{-i})] - \bE_{a_{-i} \sim \pi_{-i}^{(t)}}[Q_{\phi}^{(t)}(s,a_i,a_{-i})] \right| \nonumber\\
&\stackrel{(d)}{\leq} \|Q_{\phi}^{(t)}(s,a_i,\cdot)\|_{\infty} \, \text{d}_{\text{TV}}(\tilde{\pi}_{-i}^{(t)}(\cdot|s), \pi_{-i}^{(t)}(\cdot|s))\nonumber\\
&\stackrel{(e)}{\leq} \frac{\phi_{\max}}{1-\gamma}\text{d}_{\text{TV}}(\tilde{\pi}_{-i}^{(t)}(\cdot|s), \pi_{-i}^{(t)}(\cdot|s))\,,
\end{align}
where~$\text{d}_{\text{TV}}(\cdot,\cdot)$ is the total variation distance,  $(a)$~uses the identity~$\bar{A}_i^{(t)}(s,a_i) = \bar{A}_{\phi,i}^{(t)}(s,a_i)$, $(b)$~stems from the definitions~\eqref{eq:def-tilde-A-phi} and~\eqref{eq:def-bar-A-phi}, $(c)$~is a consequence of the definition of the advantage function, $(d)$~follows from the standard inequality~$\left| \int_{\Omega} f d\mu -  \int_{\Omega} f d\nu \right| \leq \|f\|_{\infty} \text{d}_{\text{TV}}(\mu, \nu)$ for any bounded measurable function~$f: \Omega \to \mathbb{R}$ and any probability measures~$\mu$ and~$\nu$ and finally~$(e)$ stems from observing that~$\|Q_{\phi}^{(t)}(s,a_i,\cdot)\|_{\infty} \leq \phi_{\max}/(1-\gamma)\,.$

Plugging the above inequality into Term B, we obtain 
\begin{align}
\label{eq:ineq-interm-termB}
|\text{Term B}| &\leq  \frac{1}{1-\gamma} \sum_{s \in \mathcal{S}} d_{\mu}^{(t+1)}(s) \sum_{i=1}^N \sum_{a_i \in \mathcal{A}_i} |\pi_i^{(t+1)}(a_i|s) - \pi_i^{(t)}(a_i|s)| \cdot |\tilde{A}_{\phi,i}^{(t)}(s,a_i) - \bar{A}_i^{(t)}(s,a_i)| \nonumber\\
&\leq \frac{\phi_{\max}}{(1-\gamma)^2} \sum_{s \in \mathcal{S}} d_{\mu}^{(t+1)}(s) \sum_{i=1}^N \text{d}_{\text{TV}}(\tilde{\pi}_{-i}^{(t)}(\cdot|s), \pi_{-i}^{(t)}(\cdot|s)) \cdot \|\pi_i^{(t+1)}(\cdot|s) - \pi_i^{(t)}(\cdot|s)\|_1\,.
\end{align}
Then we observe that
\begin{align}
\label{eq:interm-control-KL}
&\sum_{i=1}^N \text{d}_{\text{TV}}(\tilde{\pi}_{-i}^{(t)}(\cdot|s), \pi_{-i}^{(t)}(\cdot|s)) \cdot \|\pi_i^{(t+1)}(\cdot|s) - \pi_i^{(t)}(\cdot|s)\|_1\nonumber\\
&\stackrel{(a)}{\leq} \sum_{i=1}^N \left( \frac{1}{2\sqrt{N}} \text{d}_{\text{TV}}(\tilde{\pi}_{-i}^{(t)}(\cdot|s), \pi_{-i}^{(t)}(\cdot|s))^2 + \frac{\sqrt{N}}{2} \|\pi_i^{(t+1)}(\cdot|s) - \pi_i^{(t)}(\cdot|s)\|_1^2 \right)\nonumber\\ 
&\stackrel{(b)}{\leq} \sum_{i=1}^N \left(\frac{1}{2\sqrt{N}} \text{KL}(\tilde{\pi}_{-i}^{(t)}(\cdot|s)||\pi_{-i}^{(t)}(\cdot|s)) + \frac{\sqrt{N}}{2} \text{KL}(\pi_i^{(t+1)}(\cdot|s)||\pi_i^{(t)}(\cdot|s)) \right)\nonumber\\
&\stackrel{(c)}{\leq} \sum_{i=1}^N \left(\frac{1}{2\sqrt{N}} \text{KL}(\pi^{(t+1)}(\cdot|s)||\pi^{(t)}(\cdot|s)) + \frac{\sqrt{N}}{2} \text{KL}(\pi_i^{(t+1)}(\cdot|s)||\pi_i^{(t)}(\cdot|s)) \right)\nonumber\\
&= \frac{\sqrt{N}}{2} \text{KL}(\pi^{(t+1)}(\cdot|s)||\pi^{(t)}(\cdot|s)) + \frac{\sqrt{N}}{2} \sum_{i=1}^N \text{KL}(\pi_i^{(t+1)}(\cdot|s)||\pi_i^{(t)}(\cdot|s))\nonumber\\
&\stackrel{(d)}{=} \sqrt{N}\, \text{KL}(\pi^{(t+1)}(\cdot|s)||\pi^{(t)}(\cdot|s))\,,
\end{align}
where~$(a)$~stems from Young's inequality ( $x y \leq \frac{x^2}{2 \sqrt{N}} + \frac{\sqrt{N} y^2}{2}$ for any~$x, y \in \mathbb{R}$), $(b)$~is a consequence of Pinsker's inequality, $(c)$~follows from observing that
$$
\text{KL}(\tilde{\pi}_{-i}^{(t)}(\cdot|s)||\pi_{-i}^{(t)}(\cdot|s)) = \sum_{j < i} \text{KL}(\pi_j^{(t+1)}(\cdot|s)||\pi_j^{(t)}(\cdot|s)) \leq \sum_{j=1}^N \text{KL}(\pi_j^{(t+1)}(\cdot|s)||\pi_j^{(t)}(\cdot|s))\,,
$$
where we used the definition of~$\tilde{\pi}_{-i}^{(t)}$ in~\eqref{eq:def-tilde-pi}. 
Then, $(d)$ uses the additivity of the KL divergence for product distributions. 
Overall, plugging~\eqref{eq:interm-control-KL} into~\eqref{eq:ineq-interm-termB}, we obtain 
\begin{equation}
\label{eq:bound-termB-natural-grad}
|\text{Term B}| \leq \frac{\phi_{\max}\sqrt{N}}{(1-\gamma)^2} \sum_{s \in \mathcal{S}} d_{\mu}^{(t+1)}(s)  \text{KL}(\pi^{(t+1)}(\cdot|s)||\pi^{(t)}(\cdot|s))\,.
\end{equation}
\end{enumerate}
Combining~\eqref{eq:bound-termA-natural-grad} and~\eqref{eq:bound-termB-natural-grad} yields the desired inequality~\eqref{eq:natural-grad-policy-improvement} and concludes the proof. 
\end{proof}

We now introduce an additional convenient notation for the Nash gap  induced by a given policy~$\pi \in \Pi$ as follows: 
\begin{align}
\text{Nash-gap}_i(\pi) &\triangleq \underset{\pi_i' \in \Pi^i}{\max} V_i^{\pi_i', \pi_{-i}}(\rho) - V_i^{\pi_i, \pi_{-i}}(\rho)\,, \forall i \in \mn\,,\\
\text{Nash-gap}(\pi) &\triangleq \underset{i \in \mn}{\max}\, \text{Nash-gap}_i(\pi)\,.
\end{align}
The next lemma connects the second term in the policy improvement (see Lemma~\ref{lem:natural-grad-policy-improvement}) with the Nash equilibrium gap. 
This lemma is a slightly refined version of Lemma~21 in \cite{zhang-et-al22softmax}. 
\begin{lemma}
\label{lem:lemma21-zhang}
If~$\eta \leq (1-\gamma)^2$, then for any initial distribution~$\nu \in \Delta(\ms),$
\begin{equation}
\text{Nash-gap}(\pi^{(t)})^2 \leq \frac{3 \tilde{\kappa}_{\rho}}{c \eta^2 (1-\gamma)}\sum_{i=1}^N \sum_{s \in \mathcal{S}} d_{\nu}^{\pi^{(t+1)}}(s) \log Z_t^{i,s}\,, 
\end{equation}
where we recall that~$c \triangleq \underset{i \in \mn}{\min}\underset{t \in \bN}{\inf}\, \underset{s \in \mathcal{S}}{\min} \underset{a_i^* \in \underset{a_i \in \mai}{\argmax}\bar{Q}^{\pi^{(t)}}_i(s,a_i)}{\sum}\,   \pi_i^{(t)}(a_i^*|s) >0 \,.$
\end{lemma}

\begin{proof}
We start by upper-bounding the Nash gap of every policy~$\pi^{(t)}$ for every iteration~$t$. As a first step, we use the performance difference lemma for policy~$\pi^{(t)}$, for any~$i \in \mn$ and any policy~$\pi_i' \in \Pi^i$ to write: 
\begin{align}
V_i^{\pi_i', \pi_{-i}^{(t)}}(\rho) - V_i^{\pi^{(t)}}(\rho) 
&= \frac{1}{1-\gamma} \sum_{s \in \ms, a_i \in \mai} d_{\rho}^{\pi_i', \pi_{-i}^{(t)}}(s) \pi_i'(a_i|s) \bar{A}_i^{\pi^{(t)}}(s,a_i)\nonumber\\
&\leq \frac{1}{1-\gamma} \sum_{s \in \ms} d_{\rho}^{\pi_i', \pi_{-i}^{(t)}}(s) \underset{a_i \in \mai}{\max} \bar{A}_i^{\pi^{(t)}}(s,a_i) \nonumber\\ 
&= \frac{1}{1-\gamma} \sum_{s \in \ms} \sqrt{\frac{d_{\rho}^{\pi_i', \pi_{-i}^{(t)}}(s)}{d_{\nu}^{\pi^{(t+1)}}(s)}} \sqrt{d_{\nu}^{\pi^{(t+1)}}(s)} \sqrt{d_{\rho}^{\pi_i', \pi_{-i}^{(t)}}(s)} \underset{a_i \in \mai}{\max} \bar{A}_i^{\pi^{(t)}}(s,a_i) \nonumber\\
&\leq \frac{\sqrt{\tilde{\kappa}_{\rho}}}{(1-\gamma)^{3/2}} \sum_{s \in \ms}\sqrt{d_{\nu}^{\pi^{(t+1)}}(s)} \sqrt{d_{\rho}^{\pi_i', \pi_{-i}^{(t)}}(s)} \underset{a_i \in \mai}{\max} \bar{A}_i^{\pi^{(t)}}(s,a_i) \nonumber\\ 
&\leq \frac{(1-\gamma)\sqrt{\tilde{\kappa}_{\rho}}}{\eta(1-\gamma)^{3/2}} 
\left( \sum_{s \in \ms} d_{\rho}^{\pi_i', \pi_{-i}^{(t)}}(s) \right)^{1/2} 
\left(\sum_{s \in \ms} d_{\nu}^{\pi^{(t+1)}}(s) \left( \frac{\eta \underset{a_i \in \mai}{\max} \bar{A}_i^{\pi^{(t)}}(s,a_i) }{1-\gamma} \right)^2 \right)^{1/2}\nonumber\\ 
&\leq \frac{\sqrt{\tilde{\kappa}_{\rho}}}{\eta \sqrt{1-\gamma}} 
\left(\sum_{i = 1}^N \sum_{s \in \ms} d_{\nu}^{\pi^{(t+1)}}(s) \left( \frac{\eta \underset{a_i \in \mai}{\max} \bar{A}_i^{\pi^{(t)}}(s,a_i) }{1-\gamma} \right)^2 \right)^{1/2}\,.
\end{align}
Hence, we have shown so far that 
\begin{equation}
\label{eq:1-proof-lemma-nashgap}
\text{Nash-gap}(\pi^{(t)})^2 \leq \frac{\tilde{\kappa}_{\rho}}{\eta^2 (1-\gamma)} \sum_{i = 1}^N \sum_{s \in \ms} d_{\nu}^{\pi^{(t+1)}}(s) \left( \frac{\eta \underset{a_i \in \mai}{\max} \bar{A}_i^{\pi^{(t)}}(s,a_i) }{1-\gamma} \right)^2\,. 
\end{equation}
To conclude the proof, we now use an inequality that appeared in the proof of Lemma~21 in \cite{zhang-et-al22softmax} in p. 26 (in the arxiv version), namely the second to last inequality in the proof: 
\begin{equation}
\label{eq:2-proof-lemma-nashgap}
\sum_{i = 1}^N \sum_{s \in \ms} d_{\nu}^{\pi^{(t+1)}}(s) \left( \frac{\eta \underset{a_i \in \mai}{\max} \bar{A}_i^{\pi^{(t)}}(s,a_i) }{1-\gamma} \right)^2 
\leq \frac{3}{c} \sum_{i=1}^N \sum_{s \in \mathcal{S}} d_{\nu}^{\pi^{(t+1)}}(s) \log Z_t^{i,s}\,.
\end{equation}

Combining~\eqref{eq:1-proof-lemma-nashgap} and~\eqref{eq:2-proof-lemma-nashgap} leads to the desired inequality of Lemma~\ref{lem:lemma21-zhang}. 
\end{proof}

The last part of the proof follows the same lines as~\cite{zhang-et-al22softmax} (see p. 29 therein). 
Combining Lemma~\ref{lem:natural-grad-policy-improvement} with Lemma~\ref{lem:lemma21-zhang}, we obtain: 
\begin{equation}
\text{Nash-gap}(\pi^{(t)})^2 \leq \frac{3 \tilde{\kappa}_{\rho}}{c \eta (1-\gamma)} (\Phi^{\pi^{(t+1)}}(\nu) - \Phi^{\pi^{(t)}}(\nu))\,.
\end{equation}
As a consequence, we have 
\begin{equation}
\frac{1}{T} \sum_{t=1}^{T} \text{NE-gap}(\pi^{(t)})^2 \leq \frac{3 \tilde{\kappa}_{\rho}}{c \eta (1-\gamma) T} (\Phi^{\pi^{(T+1)}}(\nu) - \Phi^{\pi^{(1)}}(\nu)) \leq  \frac{6 \tilde{\kappa}_{\rho} \phi_{\max}}{(1-\gamma)^2 c \eta T}
\end{equation}

We conclude by using Jensen's inequality and setting the step size~$\eta$ to its largest possible value, i.e., $\eta = \frac{(1-\gamma)^2}{2\phi_{\max} \sqrt{N}}$: 
\begin{equation}
\text{Nash-regret}(T) \leq 
\left(\frac{1}{T} \sum_{t=1}^{T} \text{NE-gap}(\pi^{(t)})^2\right)^{1/2} 
\leq \sqrt{\frac{12 \phi_{\max}^2 \tilde{\kappa}_{\rho} \sqrt{N}}{(1-\gamma)^4 c T}}\,.
\end{equation}

\subsection{Auxiliary Lemmas}

The following lemma is standard and well-known in the literature (see e.g.  \cite{xiao22jmlr}). 

\begin{lemma}[Performance Difference Lemma]
\label{lem:performance-difference-single-agent}
    Consider a Markov Decision Process~$(\S, \A, P, r, \gamma, \mu)$. 
    Then for any policies~$\pi, \pi' \in \Pi,$ 
    \begin{equation*}
        V^{\pi'}(\mu) - V^{\pi}(\mu) = \frac{1}{1-\gamma}
        \sum_{s\in\ms} \sum_{a\in\ma} d_{\mu}^{\pi'}(s) (\pi'(a|s) - \pi(a|s)) Q^{\pi}(s,a)\,, 
    \end{equation*}
    where for every policy~$\pi \in \Pi, V^{\pi}(\mu)$ and~$Q^{\pi}: \ms \times \ma \to \R$ are the state-value function with initial distribution~$\mu$ and state-action value function of policy~$\pi$ respectively. 
\end{lemma}
\begin{proof}
See Lemma 1 in \cite{xiao22jmlr}.
\end{proof}

The next lemma is a multi-agent version of the above performance difference lemma which was used for example in \cite{leonardos2021global}.  
\begin{lemma}[Multi-Agent Performance Difference Lemma (Single-Agent Deviation)]
\label{lem:performance-difference-multi-agent}
    Consider a Markov game $\Gamma = \left\langle \mn, \ms, (\mai)_{i \in \mn}, P, (r_i)_{i \in \mn}, \gamma \right\rangle$ with initial state distribution $\mu\in \Delta(\ms)$. Then, for any agent~$i \in \mn$, any two policies $\pi'_i,\pi_i\in\Pi^i$ and any joint policy $\pi_{-i}\in\Pi^{-i}$ of the other agents, we have: 
    \begin{align*}
        V_{i}^{\pi'_i, \pi_{-i}}(\mu) - V_{i}^{\pi_i, \pi_{-i}}(\mu) &= \frac{1}{1-\gamma} \sum_{s\in\ms}\sum_{a_i\in\ma_i} d_{\mu}^{\pi'_i, \pi_{-i}}(s) (\pi'_i(a_i|s) - \pi_i(a_i|s)) \bar{Q}_i^{\pi_i,\pi_{-i}}(s,a_i)\,,
    \end{align*}
    where we use the same notations as in section~\ref{sec:preliminaries}. 
\end{lemma}
\begin{proof}
    We provide a proof for completeness. Using \cref{lem:performance-difference-single-agent}, we have for every~$i \in \mn, \pi'_i,\pi_i\in\Pi^i, \pi_{-i}\in\Pi^{-i},$
    \begin{align*}
        V_{i}^{\pi'_i, \pi_{-i}}(\mu) - V_{i}^{\pi_i, \pi_{-i}}(\mu) &= \frac{1}{1-\gamma} \sum_{s\in\ms} \sum_{a = (a_i, a_{-i})\in\ma} d_{\mu}^{\pi_i',\pi_{-i}}(s) \left( \pi_i'(a_i|s) \pi_{-i}(a_{-i}|s) - \pi_i(a_i|s) \pi_{-i} (a_{-i}|s) \right) Q_i^{\pi}(s,a)\\
        &= \frac{1}{1-\gamma} \sum_{s\in\ms} \sum_{a_i\in\ma_i} d_{\mu}^{\pi_i',\pi_{-i}}(s) \left( \pi'_i(a_i|s) - \pi_i(a_i|s) \right) \underbrace{\sum_{a_{-i}\in\ma_{-i}} \pi_{-i}(a_{-i}|s) Q_i^{\pi}(s,(a_i,a_{-i}))}_{=\bar{Q}^{\pi}_i(s,a_i)} \\
        &= \frac{1}{1-\gamma} \sum_{s\in\ms} \sum_{a_i\in\ma_i} d_{\mu}^{\pi'_i, \pi_{-i}}(s) (\pi'_i(a_i|s) - \pi_i(a_i|s)) \bar{Q}_i^{\pi}(s,a_i).
    \end{align*}
\end{proof}

\end{document}